\documentclass{article}

\usepackage[main,final]{neurips_2026}
\usepackage[T1]{fontenc}

\usepackage{microtype}
\usepackage{graphicx}
\usepackage{subcaption}
\usepackage{booktabs} 
\usepackage{hyperref}
\usepackage{url}

\usepackage{apxproof}
\usepackage{amsmath}
\usepackage{amssymb}
\usepackage{mathtools}
\usepackage{amsthm}

\usepackage{amsfonts}  
\usepackage{commath}   
\usepackage{cancel}    
\usepackage{bm}        
\usepackage{xcolor}    
\usepackage{etoolbox}  
\usepackage{comment}   
\usepackage{mdframed}  
\usepackage{xspace}    
\usepackage{siunitx}   
\usepackage{float}     

\usepackage{doi}          

\newtheorem{problem}{Problem}

\newtheorem{example}{Example}

\DeclareMathOperator{\expect}{\mathbb{E}}
\DeclareMathOperator{\probability}{\mathbb{P}}
\DeclareMathOperator{\Var}{\operatorname{Var}}
\DeclareMathOperator{\Cov}{\operatorname{Cov}}

\usepackage[capitalize,noabbrev]{cleveref}

\theoremstyle{plain}
\newtheorem*{informal*}{Informal Theorem}
\newtheorem{theorem}{Theorem}[section]

\newtheorem{lemma}[theorem]{Lemma}

\theoremstyle{definition}
\newtheorem{definition}[theorem]{Definition}

\theoremstyle{remark}
\newtheorem{remark}[theorem]{Remark}

\usepackage[textsize=tiny]{todonotes}
\setuptodonotes{inline}

\title{Exact Gaussian Moment Matching for\\ Residual Networks: a Second-Order Method}

\author{%
  Simon Kuang and Xinfan Lin \\ 
  Department of Mechanical and Aerospace Engineering\\
  University of California, Davis\\
  Davis, CA 95616 \\
  \texttt{\{slku, lxflin\}@ucdavis.edu} \\
}

\counterwithin{figure}{section}
\counterwithin{table}{section}

\usepackage{multirow}

\DeclareMathOperator{\normal}{\mathrm N}

\newenvironment{after}{}{}

\begin{document}
\maketitle

\begin{abstract}
    We study the problem of propagating the mean and covariance of a general multivariate Gaussian distribution through a deep (residual) neural network using layer-by-layer moment matching.
    We close a longstanding gap 
    by deriving exact moment matching for the probit, GeLU, ReLU (as a limit of GeLU), Heaviside (as a limit of probit), and sine activation functions;
    for both feedforward and generalized residual layers.
    On large-variance random networks, median KL-divergence improvements over common baselines range from tens to thousands of times.
    On a variational Bayes neural network, we show that our method improves upon a state-of-the-art deterministic inference method by an order of magnitude in KL divergence from Monte Carlo ground truth.
    We also give a smooth-distance error bound showing that, under regularity assumptions, moment matching removes the leading low-variance errors and propagates higher-order local accuracy through the layers of a network.
\end{abstract}

\section{Introduction}
Given a Gaussian input distribution, what is the output distribution of a neural network?
This fundamental problem arises in robustness analysis of a network under typical perturbations \citep{wright_analytic_2024},
making predictions on noisy data using a neural network trained on noiseless data \citep{bibi_analytic_2018},
computing the loss function of a variational Bayes neural network \citep{frey_variational_1999,wu_deterministic_2019,petersen_uncertainty_2024, wright_analytic_2024,rui_li_streamlining_2025},
and understanding the wide limit of randomly initialized deep networks \citep{he_delving_2015}.
Concurrent work views this problem as a goalpost for mechanistic understanding of neural networks, with a view towards risk estimation in AI safety \citep{wu_estimating_2026}.

\begin{table*}
    \begin{center}
      \begin{tabular}{ll}
      \toprule
      Assumption & References \\
      \midrule
      \(\Sigma\) is small (linearized)  & \citet{titensky_uncertainty_2018, nagel_kalman-bucy-informed_2022}\\
      & \citet{petersen_uncertainty_2024, jungmann_analytical_2025} \\
      & 
        \citet{bergna_post-hoc_2025,rui_li_streamlining_2025}
      \\
      \(\Sigma\) is small (unscented) & \citet{astudillo_propagation_2011, abdelaziz_uncertainty_2015} \\
      \(\Sigma\) is diagonal          &  \citet{abdelaziz_uncertainty_2015,huber_bayesian_2020}\\
      & 
      \citet{goulet_tractable_2021, rui_li_streamlining_2025}
      \\
      & \citet{wagner_kalman_2022,akgul_deterministic_2025} \\
      \(\mu = 0\)    & \citet{cho_kernel_2009,bibi_analytic_2018}
      \\ 
      & NNGP literature \citep[Table 1]{han_fast_2022} \\
      \(\mu \to \infty\) & \citet{wu_deterministic_2019} \\
      \textbf{no assumptions}                 &  \citet{wright_analytic_2024}\\
      & this paper \\
      \bottomrule
      \end{tabular}
    \end{center}
    
    \caption{\label{tab:covariance-assumptions} Comparison of assumptions imposed on Gaussian approximations of neural network layers with input \(\mathcal N (\mu, \Sigma)\).}
\end{table*}

  The baseline method is Monte Carlo, whose sample mean and covariance are guaranteed to be asymptotically exact in the limit of infinitely many samples.
  But Monte Carlo is computationally expensive and scales poorly with input dimension \citep{tang_note_2024}.
  It is for this reason that all of the works cited above apply deterministic approximations to the mean and covariance of neural network outputs.
  Despite the fact that exactly evaluating a network's pushforward distribution is NP-hard, 
  moment matching---which approximates the distribution of a hidden layer as Gaussian in order to derive a closed form for the moments of the activation function---has proven practically effective.

  The main subtasks of moment matching are:
  (i) compute the mean vector $\mu$ and variance-covariance matrix $\Sigma$ of any hidden layer whose pre-activations have an arbitrary multivariate Gaussian distribution,
  and
  (ii) certify that the true post-activation distribution is approximately Gaussian.
  
  Literature on moment matching in neural networks has responded to objective (i) by restricting the mean and covariance of the input distribution.
  Thus, the methodological and theoretical states of the art are inexact on networks where the hidden layer has more than one neuron, and omit the linear (residual) bypass architecture altogether.
  Given that many neurons are needed to learn arbitrary nonlinear functions, the prevailing assumptions severely limit the scope of moment matching.

In this work, we address both challenges.
We derive the exact closed-form, analytical expressions for the means and covariances of both feedforward and
residual layers, and certify that on multi-layer networks, moment matching is provably asymptotically more accurate than the delta method.

\paragraph{Contribution}
For an unrestricted multivariate Gaussian input, we derive closed-form mean and covariance matching for feedforward and residual layers with probit, GeLU, ReLU, Heaviside, and sine activation functions.
The derivations rely on novel techniques involving integral and differential identities of the Gaussian function, 
resolving (i).
We then prove that exact moment matching cancels the leading terms in a smooth fourth-order probability distance, and that the resulting higher-order local accuracy composes across layers, resolving (ii).

\paragraph{Related work}
Existing results on distribution propagation
restrict the parameters of the Gaussian input distribution (Table~\ref{tab:covariance-assumptions}).
%
%
\begin{table*}
    \begin{center}
      \begin{tabular}{ll}
      \toprule
      Activation function & References \\
      \midrule
      piecewise linear & \citet{cho_kernel_2009,bibi_analytic_2018}\\
      & \citet{huber_bayesian_2020,wright_analytic_2024}\\
      & \citet{akgul_deterministic_2025} \\
      logistic (\(\approx\) piecewise exponential) &  \citet{astudillo_propagation_2011,abdelaziz_uncertainty_2015}\\
      logistic (\(\approx\) \(\Phi\)) & \citet{huber_bayesian_2020} \\
      Heaviside &  \citet{wu_deterministic_2019, wright_analytic_2024}\\
      GeLU & \citet{wright_analytic_2024}\\
      (ReLU, Heaviside) \textbf{+ affine} & \citet{wu_deterministic_2019} \\
      (\(\sin\), \(\Phi\), GeLU, ReLU, Heaviside) \textbf{+ affine} & this paper (\textbf{exact}) \\
      \bottomrule
      \end{tabular}
    \end{center}
    \caption{\label{tab:activation-functions} Activation functions for which moment propagation has been approximated.
    Note two distinct approaches to approximating the logistic function.}
\end{table*}
For certain activation functions including ReLU and GeLU,
the mean and the diagonal of the covariance matrix can be computed explicitly,
but (prior to this paper) there is no closed-form known for off-diagonal covariances of a hidden layer.
Some works take the mean-field assumption by setting them to zero \citep{huber_bayesian_2020,goulet_tractable_2021,wagner_kalman_2022,rui_li_streamlining_2025,bergna_post-hoc_2025,akgul_deterministic_2025,rui_li_streamlining_2025}.
\citet{bibi_analytic_2018} uses an analytical approximation around zero mean, and \citet{wu_deterministic_2019} uses an analytical approximation around infinite mean.
For the logistic activation function \(\sigma(x) = (1 + e^{-x})^{-1}\);
\citet{astudillo_propagation_2011,abdelaziz_uncertainty_2015}, and \citet{huber_bayesian_2020} approximate \(\sigma\) with another function having closed-form Gaussian moments,
such as a piecewise exponential function or a rescaled Normal CDF \(\Phi\).
Table~\ref{tab:activation-functions} catalogs the literature on moment approximations for activation functions.
We exemplify the failure modes of the above methods by giving single-hidden-layer counterexamples in \S\ref{sec:adversariality}.

In these prior works, the primary reason for taking the distributional restrictions is to get around the apparent intractability of \(\mathbb R^2\) improper integrals which arise in moment matching.
For example, \citet{wu_deterministic_2019} expands the covariance to second order around degenerate cases \(\mu=0\) or \(\mu=\infty\).
The \(\mu=0\) case, frequently encountered in neural network Gaussian Process (NNGP) literature \citep{han_fast_2022}, allows the simplification of two-dimensional integrals using polar coordinates.
For a general activation function, \citet{wright_analytic_2024} uses a Fourier transform to derive the exact mean and covariance matrix of the output as a formal power series in inter-neuron correlation \(\rho \approx 0\). 
The series, however, must be truncated, as each coefficient needs to be derived algebraically by hand.
In this paper, we eschew the \(\mathbb{R}^2\) setup altogether by working with probabilistic properties of Gaussian measure, such as the Stein's integration-by-parts identity and the Gaussian ODE \(\phi'(x) = -x \phi(x)\).

The closest theoretical result is \citet{petersen_uncertainty_2024}, which establishes optimal distributional approximation for a single ReLU neuron.
In contrast, our theoretical guarantee applies jointly to all of the neurons in a layer, and all of the layers in a network.

Our work is also the first exact moment matching in residual networks, which requires resolving cross-covariances between the pre-activation and the activation.
The closest prior work is \citet[Appendix C.1]{wu_deterministic_2019}, which uses asymptotic approximations.


\paragraph{Roadmap}
\S\ref{sec:methodology} states the full uncertainty-propagation problem (Problem~\ref{prob:problem}) and the reduced single-layer Gaussian moment matching problem (Problem~\ref{prob:problem-gaussian}), and then resolves the latter with exact mean and covariance formulas for feedforward and residual layers (Lemma~\ref{lem:moment-maps}).
\S\ref{sec:adversariality} clarifies the limitations of various Gaussian moment propagation methods, including our own, through example problems.
\S\ref{sec:theoretical-guarantees} applies these formulas to address the full multilayer problem by recursive moment matching over layers and bounds the error in a smooth fourth-order integral probability metric.
\S\ref{sec:examples} validates the moment matching accuracy on single- and multilayer random networks, applies them to variational Bayes neural networks, and sketches an extension to stochastic activations.

\section{Exact, closed-form analytical expressions for mean and covariance}
\label{sec:methodology}


A neural network is a composition of layer functions \(\mathbb R ^n \to \mathbb R ^m\)
\begin{equation}
    g(x; A, b, C, d) = \sigma(A x + b) + C x + d, \label{eq:layer-function}
\end{equation}
where \(A \in \mathbb R^{m \times n}, b \in \mathbb R^m, C \in \mathbb R^{m \times n}, d \in \mathbb R^m\) are parameters;
the activation function of a neural network is denoted by \(\sigma:\mathbb R \to \mathbb R\) and applies elementwise.
Except for parameters \(A, b, C, d\), capital letters refer to random variables.
The layers of a neural network are indicated by superscripts, e.g.~\(A^k\) is a matrix of parameters for the \(k\)th layer.
If \(X\) is a square-integrable random vector, the notation \(\normal X\) refers to a random variable distributed as \(\mathcal N(\expect X, \Cov X)\).


\begin{definition}
    \label{def:neural-network}
    A neural network with \(\ell \) layers is the function \(f: \mathbb R^{n_x} \to \mathbb R^{n_y}\) defined by
    \begin{align*}
        f(x) &= f^\ell(x) \\
        f^k(x) &= g(f^{k-1}(x); A^k, b^k, C^k, d^k) & \forall k &\in \cbr{1 \ldots \ell} \\
        f^0(x) &= x
    \end{align*}
\end{definition}

Stated formally, the problem of uncertainty propagation studied in this paper is:
\begin{problem}
    \label{prob:problem}
    Let \(f\) be a neural network with \(\ell\) layers.
    Given \(X \sim \mathcal N(\mu, \Sigma)\), characterize the distribution of \(Y_\mathrm{true} = f(X)\).
\end{problem}
After layer-wise Gaussian approximation, this problem reduces to:
\begin{problem}
  \label{prob:problem-gaussian}
  Given \(X \sim \mathcal N(\mu, \Sigma)\) and \(A, b, C, d\);
  find exact expressions for \(\expect{g(X; A, b, C, d)}\) and \(\Cov g(X; A, b, C, d)\).
\end{problem}

\subsection{Our analytic moment-matching method \(Y_\mathrm{ana}\)}
In moment matching, we re-approximate each layer by a Gaussian sharing its first two moments:
\begin{definition}
    Let \(f\) be a neural network with \(\ell\) layers.
    Given \(X \sim \mathcal N(\mu, \Sigma)\), the moment-matching Gaussian approximation of \(f(X)\), 
    is the random variable \(Y_\mathrm{ana}\) defined by
    \begin{align*}
        Y_\mathrm{ana} &=  Y^\ell\\
        Y^k &= \normal g(Y^{k-1}; A^k, b^k, C^k, d^k) & \forall k &\in \cbr{1 \ldots \ell} \\
        Y^0 &= X
    \end{align*}
\end{definition}

We specify three transcendental functions needed to compute the first two Gaussian moments of a layer defined by \eqref{eq:layer-function}.
\begin{definition}
  \label{def:moment-maps}
  Given a nonlinear function \(\sigma: \mathbb{R} \to \mathbb R\),
  the functions \(M_\sigma: \mathbb{R}  \times \mathbb{R}_+ \to
  \mathbb{R}\) and \(K_\sigma, L_\sigma: \mathbb{R}^2 \times \mathbb
  R_{\geq 0}^{2 \times 2} \to \mathbb{R}\) are
  \begin{gather*}
    M_\sigma(\mu; \nu) = \expect{\sigma(X)},
    \qquad
     X \sim \mathcal N(\mu, \nu);
  \end{gather*}
  \begin{align*}
    K_\sigma(\mu_1, \mu_2; \nu_{11}, \nu_{22}, \nu_{12})
    &= \Cov(\sigma(X_1), \sigma(X_2)),
    &
    \begin{pmatrix} X_1 \\ X_2 \end{pmatrix}
    &\sim \mathcal N\del{
      \begin{pmatrix} \mu_1 \\ \mu_2 \end{pmatrix},
      \begin{pmatrix} \nu_{11} & \nu_{12} \\ \nu_{12} & \nu_{22} \end{pmatrix}
    };
  \end{align*}
  and
  \begin{align*}
    L_\sigma(\mu_1; \nu_{11}, \nu_{22}, \nu_{12})
    &= \Cov(\sigma(X_1), X_2),
    &
    \begin{pmatrix} X_1 \\ X_2 \end{pmatrix}
    &\sim \mathcal N\del{
      \begin{pmatrix} \mu_1 \\ \star \end{pmatrix},
      \begin{pmatrix} \nu_{11} & \nu_{12} \\ \nu_{12} & \nu_{22} \end{pmatrix}
    }.
  \end{align*}
\end{definition}

Approximating or computing these functions enables moment matching:
\begin{lemma}
\label{lem:moment-maps}
  For some activation function \(\sigma\), let \(g_\sigma(x; A, b, C, d) = \sigma(Ax + b) + Cx + d\).
  Let \(X \sim \mathcal N(\mu, \Sigma)\).
  Then
  \begin{align*}
    \del{\expect g_\sigma(X; A, b, C, d)}_i
    &= M_\sigma(\mu_i; \nu_{ii}) + (C\mu)_i + d_i
  \end{align*}
  and
  \begin{multline*}
      \del{\Cov g_\sigma(X; A, b, C, d)}_{i, j}
      =
        K_\sigma\del{\mu_i, \mu_j; \nu_{ii}, \nu_{jj}, \nu_{ij}}
        \\
        + L_\sigma\del{\mu_i; \nu_{ii}, \tau_{jj},\kappa_{ij}}
        + L_\sigma\del{\mu_j; \nu_{jj}, \tau_{ii},\kappa_{ji}}
        + \tau_{ij},
  \end{multline*}
  where for all valid indices \((i, j)\),
  \begin{align*}
    \mu_i &= (A\mu + b)_i
    &
    \tau_{ij}
    &= (C\Sigma C^\intercal)_{i,j}
    \\
    \nu_{ij} &= (A\Sigma A^\intercal)_{i,j}
    &
    \kappa_{ij}
    &= (A \Sigma C^\intercal)_{i,j}
  \end{align*}
\end{lemma}
Given that their domain is five-dimensional, these functions are not amenable to tabulation, and are typically approximated in the state of the art. 
In this work, instead, we compute them analytically for hidden layers with the following activation functions:
\begin{description}
  \item[probit] in App.~\ref{app:probit} by expressing \(\Phi(x) = \expect\sbr{Z \leq x}\) in terms of an auxiliary standard Normal \(Z\),
  generalizing \citet[\S~III.B]{huber_bayesian_2020} and \citet[App.~B.4]{wright_analytic_2024}.
  \item[GeLU] in App.~\ref{app:gelu} by repeated applications of the multivariate Stein's lemma and the Gaussian ODE \(\phi'(x) + x\phi(x)= 0\), generalizing \citet[App.~B.2]{wright_analytic_2024}.
  \item[ReLU] in App.~\ref{app:relu} by using the Dominated Convergence Theorem to take the pointwise limit 
  \begin{gather*}
  \operatorname{ReLU}(x) = \lim_{\lambda \to \infty} \lambda^{-1} \operatorname{GeLU}(\lambda x),
  \end{gather*}
  generalizing \citet[App.~C.3]{frey_variational_1999}, \citet[App.~A]{cho_kernel_2009}, \citet[App~A.2.2]{wu_deterministic_2019}, \citet[\S~III.C]{huber_bayesian_2020}, and \citet[App.~B.2]{wright_analytic_2024}.
  \item[Heaviside] in App.~\ref{app:heaviside} by using the Dominated Convergence Theorem to take the pointwise limit 
  \begin{gather*}
  \operatorname{Heaviside}(x) = \lim_{\lambda \to \infty} \Phi(\lambda x),
  \end{gather*}
  generalizing \citet[App.~C.2]{frey_variational_1999}, \citet[App.~A.2.1]{wu_deterministic_2019}, and \citet[App.~B.2]{wright_analytic_2024}.
  \item[sine] in App.~\ref{app:sine} by combining the characteristic function of the Normal distribution with the trigonometric identity \(\sin(x) = \del{e^{ix} - e^{-ix}} /(2i)\), generalizing \citet[App.~1]{sitzmann_implicit_2020}.
\end{description}
The calculations are interesting in themselves because they are probabilistic in nature:
we never resort to Riemann integrals against the Gaussian density.
Instead we only use the fact that Gaussian density satisfies the ordinary differential equation \(\phi'(x) = -x\phi(x)\) and the closely related Stein's identity \(\expect X f(Y) = \Cov(X, Y) \expect f'(Y)\) for jointly Normal \((X, Y)\) and differentiable \(f\), and derive from these a panoply of Gaussian integral identities in App.~\ref{app:gaussian-integrals} which are ultimately used to solve all of the required integrals.
To convey the flavor of these derivations,
we list the most original result:
{\footnotesize
\begin{multline*}
  K_\text{GeLU}(\mu_1, \mu_2; \nu_{11}, \nu_{22}, \nu_{12})
    \\
    = 
    \del{
      \mu_1 \nu_{12} + \mu_{2}\nu_{11}  - \frac{\mu_{1} \nu_{12} \nu_{11}}{
        1 + \nu_{11}
      }
    }
    \frac{1}{\sqrt{1 + \nu_{11}}}
    \Phi_{2;1}
    \del{
      \frac{\mu_1}{\sqrt{1 + \nu_{11}}}
      , \frac{\mu_2}{\sqrt{1 + \nu_{22}}};
      \frac{\nu_{12}}{\sqrt{(1 + \nu_{11})(1 + \nu_{22})}}
    }
    \\
    + \del{
      \mu_2 \nu_{12} + \mu_{1}\nu_{22} - \frac{\mu_{2} \nu_{12} \nu_{22}}{
        1 + \nu_{22}
      }
    }
    \frac{1}{\sqrt{1 + \nu_{22}}}
    \Phi_{2;1}
    \del{
      \frac{\mu_2}{\sqrt{1 + \nu_{22}}}
      , \frac{\mu_1}{\sqrt{1 + \nu_{11}}};
      \frac{\nu_{12}}{\sqrt{(1 + \nu_{11})(1 + \nu_{22})}}
    }
    \\
    + 
    \frac{
      \nu_{11} \nu_{22} + \nu_{12}^2
      \del{
        1 - \frac{\nu_{11}}{1 + \nu_{11}}
        - \frac{\nu_{22}}{1 + \nu_{22}}
      }
    }{\sqrt{(1 + \nu_{11})(1 + \nu_{22})}}
    \phi_{2}\del{
      \frac{\mu_1}{\sqrt{1 + \nu_{11}}},
      \frac{\mu_2}{\sqrt{1 + \nu_{22}}};
      \frac{\nu_{12}}{\sqrt{(1 + \nu_{11})(1 + \nu_{22})}}
    }
    \\
    + \del{\mu_1\mu_2 + \nu_{12}}
    \Phi_2\del{
      \frac{\mu_1}{\sqrt{1 + \nu_{11}}},
      \frac{\mu_2}{\sqrt{1 + \nu_{22}}};
      \frac{\nu_{12}}{\sqrt{(1 + \nu_{11})(1 + \nu_{22})}}
    }
    -M_\text{GeLU}(\mu_1; \nu_{11}) M_\text{GeLU}(\mu_2; \nu_{22}).
\end{multline*}
}

\subsection{Targets: ground truth \(Y_\mathrm{true}\) and pseudo-truth \(Y_\mathrm{pseudo}\)}
We benchmark our analytic approach and baseline methods (listed in Appendix~\ref{sec:baselines})
against the \textbf{true distribution}
\(Y_\mathrm{true} = f(X)\)
and the \textbf{pseudo-true distribution}
\(Y_\mathrm{pseudo} = \mathcal N\del{\expect Y_\mathrm{true}, \Cov Y_\mathrm{true}}\).
While \(Y_\mathrm{true}\) is the ideal answer to Problem~\ref{prob:problem}, \(Y_\mathrm{pseudo}\) is the closest (by KL divergence) Gaussian approximation to \(Y_\mathrm{true}\).



\section{Explicit examples}
\label{sec:adversariality}
In order to motivate our theoretical results, we present some explicit examples which both show why analytic moment propagation is necessary and define its fundamental limitations.
First, linear, unscented, and mean-field propagation (baselines presented in \cref{sec:baselines}) can be arbitrarily wrong:
\begin{example}[for linear propagation]
  Consider the network \(Y = \sin(X)\), where \(X \sim \mathcal N ({0, \sigma^2})\).
  Then \(Y_\mathrm{lin} = \mathcal{N}(0, \sigma^2)\).
  But in fact \(\Var Y = (1 - e^{-2\sigma^2})/2\) which tends to \(1/2\) for large \(\sigma^2\).
  So by increasing \(\sigma^2\), we can make \(Y_\mathrm{lin}\) arbitrarily wrong.
\end{example}
\begin{example}[for unscented propagation]
  Suppose that \(X \sim \mathcal N(0, 1)\), and the sigma points are \(X \in \{-\alpha, 0, \alpha\}\) for some \(\alpha > 0\).
  Then on the neural network \(Y = \sin(\alpha^{-1}\pi X)\), \(Y_\mathrm{u95}\) and \(Y_\mathrm{u02}\) will be identically zero and arbitrarily wrong.
\end{example}
\begin{example}[for mean-field propagation]
  \label{ex:mean-field}
  Consider the following (linear) network, with scalar input \(X\), hidden \(Y^1\in \mathbb{R}^m\), and scalar output \(Y\):
  \begin{align*}
    Y &= \frac{1}{m} \sum_{i=1}^m Y^1_i,
    &
    Y^1_i &= X,
    &
    X &\sim \mathcal N(0, 1).
  \end{align*}
  The mean-field approximation treats each \(Y_i^1\) as independent \(\mathcal N(0, 1)\), so it concludes \(Y\sim \mathcal N(0, m^{-1})\).
  But \(Y\) is identical to \(X\sim \mathcal N(0, 1)\).
  So by increasing \(m\), we can make \(Y_\mathrm{mfa}\) arbitrarily wrong.
\end{example}
Whereas mean-field propagation is exact only for single neurons,
the mean and variance of our method \(Y_\mathrm{ana}\) are exact 
  on all single-hidden-layer networks, which includes all of the examples above.
But there is no free lunch (moment matching is \(\sharp P\)-hard), so we cannot expect it to be exact for all multi-layer networks.
In order to understand our method's limitations, we push it past the breaking point by constructing an explicit multi-layer network combining strong non-normality with strong nonlinearity.

\begin{example}[for \(Y_\mathrm{ana}\)]
  Consider the following network, where input \(X\), output \(Y\), and hidden \(Y^1\) are scalars; \(u(x) = \bm{1}_{x \geq 0}\) is the Heaviside function; and \(\alpha\) is a weight:
  \begin{align*}
    Y &= \alpha u(Y^1 - 3),
    &
    Y^1 &= 2 u(X),
    &
    X &\sim \mathcal N(0, 1).
  \end{align*}
  At the hidden layer, \(Y^1 \) is approximated by \(\mathcal N(1, 1)\).
  Therefore \(Y_\mathrm{ana}\) is approximated by some nondegenerate Normal distribution, scaled by \(\alpha\).
  By increasing \(\alpha\), we can make \(Y_\mathrm{ana}\) arbitrarily wrong:
  \(Y_\mathrm{true}\) is identically zero because \(Y^1 - 3 \) is always negative.  
\end{example}

\section{Second-order accuracy of moment matching}
\label{sec:theoretical-guarantees}
We quantify the error from the Gaussian Ansatz using a smooth integral probability metric.
For \(h:\mathbb R^d\to\mathbb R\), define the 
the homogeneous Sobolev seminorm
\begin{align*}
  \left\|h\right\|_{\dot W^{4,\infty}}
  :=
  \max_{1\leq j\leq 4}
  \sup_{x\in\mathbb R^d}
  \|D^j h(x)\|_{\mathrm{op}}.
\end{align*}
\begin{definition}[Fourth-order smooth distance]
\label{def:smooth-distance}
For random variables \(U,V\) taking values in \(\mathbb R^d\), define
\begin{align*}
  d_4(U,V)
  :=
  \sup_{\left\|h\right\|_{\dot W^{4,\infty}}\leq 1}
  \left|
  \expect h(U)-\expect h(V)
  \right|.
\end{align*}
\end{definition}

This is the integral probability metric generated by the homogeneous unit ball of \(\dot W^{4,\infty}\), equivalently smooth test functions with uniformly bounded derivatives of orders one through four.
The smoothness assumption covers the probit, GeLU, and sine activations.

\begin{theorem}[Single-layer higher-order accuracy]
\label{thm:smooth-single-layer}
Let \(X=\mu_X+\lambda^{-1}\xi\), where \(\xi\sim\mathcal N(0,\Sigma_X)\), and let \(Y=f(X)\).
Assume that \(f:\mathbb R^n\to\mathbb R^d\) has bounded derivatives of orders \(1,\ldots,4\).
Let \(\gamma_Y\sim\mathcal N(\expect Y,\Cov Y)\) be the moment-matched Gaussian.
Then
\begin{align*}
  d_4(Y,\gamma_Y)
  =
  O(\lambda^{-4}).
\end{align*}
By contrast, the linearized Gaussian generically has \(O(\lambda^{-2})\) error in the same smooth distance because it misses the second-order mean correction.
\end{theorem}


The single-layer theorem is the local statement behind recursive moment matching.
To propagate it through the network, we need the analogue of the Lipschitz property of Wasserstein distance.
To propagate this result through the moment-matching recursion, we 
define the fourth-order composition constant of a \(C^4\) map \(g\):
\(
  C_4(g)
  :=
  \sup_{\left\|h\right\|_{\dot W^{4,\infty}}\leq 1}
  \left\|h\circ g\right\|_{\dot W^{4,\infty}}
\).
By the chain rule and the Fa\`a di Bruno formula applied to \(g(x;A,b,C,d)=\sigma(Ax+b)+Cx+d\), \(C_4(g)\) is bounded by a polynomial in \(\|A\|\), \(\|C\|\), and \(\left\|\sigma\right\|_{\dot W^{4, \infty}}\).

\begin{theorem}[Recursive smooth-distance error bound]
\label{thm:recursive-smooth-bound}
Let \(Y_\mathrm{true}^k\) and \(Y^k\) be true and moment-matched distributions at layer \(k\) in a multi-layer network.
Assume that \(\sigma\) has bounded derivatives of orders \(1,\ldots,4\).
If \(X\) has covariance \(O(\lambda^{-2})\), then \(d_4(Y_\mathrm{true}^k,Y^k)=O(\lambda^{-4})\) for all \(k \in [1\ldots \ell]\).
\end{theorem}
\begin{proof}
The proof is the triangle inequality for the integral probability metric \(d_4\), followed by the composition bound
\begin{align*}
  d_4(g(U),g(V))
  \leq
  C_4(g)d_4(U,V).
\end{align*}
The details are given in Appendix~\ref{app:theoretical-guarantees}.
\end{proof}

Theorems~\ref{thm:smooth-single-layer} and~\ref{thm:recursive-smooth-bound} formalize the main accuracy claim.
Exact moment matching cancels the first two smooth-test terms at each layer, so its local Gaussianization error is higher order than the generic error of perturbative (linear and unscented) approximations that do not exactly match the true layer moments.
Fig.~\ref{fig:deep-sine-convergence} confirms that analytic moment matching converges quadratically, rather than linearly, in the input covariance.
The recursion shows that this higher-order local accuracy is not merely a single-layer phenomenon: it is chained through the full network, with amplification controlled by explicit fourth-order composition constants depending on the weights and activation regularity.

\begin{figure}
  \centering
  \includegraphics[width=\linewidth]{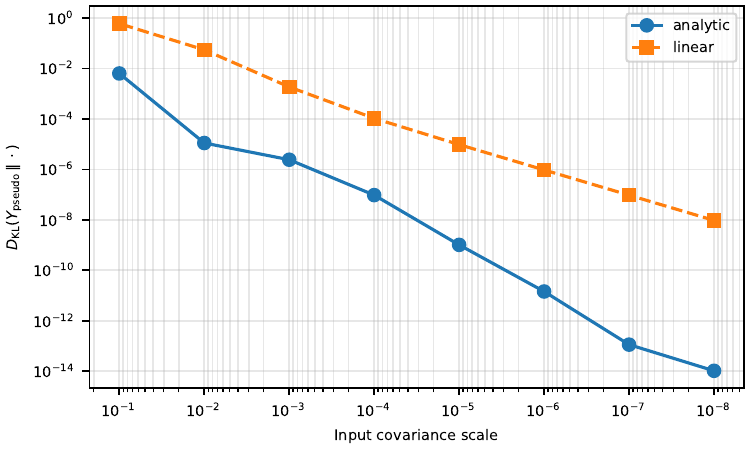}
  \caption{\label{fig:deep-sine-convergence}Empirical small-covariance convergence for a fixed deep sine network. Methodology is described in Appendix~\ref{app:deep-sine-convergence}.}
\end{figure}

\section{Examples and applications}
\label{sec:examples}
In numerical results, we obtain \(Y_\mathrm{true}\) and \(Y_\mathrm{pseudo}\) by quasi-Monte Carlo simulation.
We report the KL divergence from \(Y_\mathrm{pseudo}\) as a measure of how accurately our method computes means and covariances.
Moreover, we report the Wasserstein distance from \(Y_\mathrm{true}\) in order to assess whether our method's Gaussian Ansatz at the output layer is close to the true, non-Gaussian distribution.
The Wasserstein distance is a metric on probability distributions whose induced topology is strictly stronger than weak convergence, and we use it because it can be computed efficiently in one dimension.
\begin{definition}
Let \(X\) and \(Y\) be random variables taking values  \(\mathbb{R}^n\).
The Wasserstein distance \(d_\mathrm{W} (X, Y)\) is 
\begin{align*}
  d_\mathrm{W}(X, Y)
  &=
  \sup_{\left\|\nabla h\right\|_{\infty} \leq 1} \expect (h(X) - h(Y)).
\end{align*}

\end{definition}
\subsection{Random networks}
\label{sec:random-networks}


We apply our method and other benchmarks to 36 different ensembles of random neural networks with more than one hidden layer.
They are designed to cover a range of architectures, weights, and activations, and to stress-test the assumptions of layer-by-layer moment matching.
We sample one neural network from each ensemble and evaluate the goodness of approximation of the output distribution for three input distributions.
In each case we evaluate baselines (\cref{sec:baselines}) \(Y_\mathrm{true}\), \(Y_\mathrm{pseudo}\), \(Y_\mathrm{ana}\), \(Y_\mathrm{mfa}\), \(Y_\mathrm{lin}\), \(Y_\mathrm{u95}\), and \(Y_\mathrm{u02}\), and compare distance to the true distribution \(d_\mathrm W(\cdot, Y_\mathrm{true})\) and KL divergence from the pseudo-true distribution \(d_\mathrm{KL}(\cdot, Y_\mathrm{pseudo})\).

\begin{figure}
\centering
  \includegraphics[width=.48\linewidth]{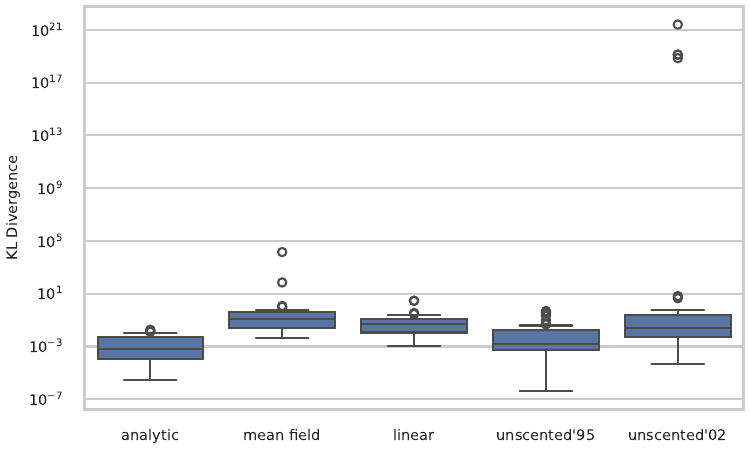}
  \hfill
  \includegraphics[width=.48\linewidth]{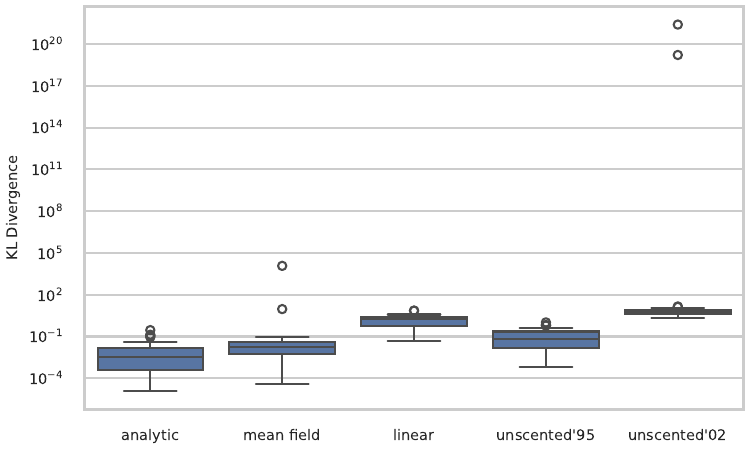}
\caption{\label{fig:kl_boxplot}Comparison of goodness of approximation (lower KL divergence is better) for all random neural networks, grouped by approximation method, in the small (left) and large (right) input variance scenarios. Four non-finite results are omitted in each scenario.}
\end{figure}

Specifications of the 108 test cases and the Monte Carlo methodology are in Appendix~\ref{app:random-neural-networks}; the individual tables and distribution plots are in the accompanying \emph{Supplementary Numerical Results}.
For large input variance, the median atio of baseline to our method's KL divergence is about 33 for Unscented'95, 340 for linearization (over the 32 cases with finite results), and 2,000 for Unscented'02, as seen in Fig.~\ref{fig:kl_boxplot}.
Even at small input variance, where linearization is justified by the delta method, our method has a lower KL divergence than linearization in all 32 cases with finite linearization results, as seen in Fig.~\ref{fig:kl_boxplot}.
The medium-variance case is visualized in Appendix~\ref{sec:random-neural-networks-summaries}.


Many examples, such as the small-variance test case with wide architecture, trained weights, and sine activation, exhibit the underdispersion of the mean-field approximation predicted in Example~\ref{ex:mean-field}.
It is noteworthy that even in the small-variance regime, where our method, linearization, and the unscented transformations are first-order asymptotically equivalent, our method is often still closest to the pseudo-true distribution.
This confirms our theoretical claim that moment matching is a second-order accurate method, while linearization is only first-order.

\subsection{Weight uncertainty: variational Bayesian neural networks}
\label{sec:bayesian-networks}
We apply our distributional approximation to variational inference for Bayesian neural networks.
In this case, the input \(x\) is deterministic and the weights \(w\) are random with a prior distribution \(p(w)\).
The network predicts \(y \sim p(y \mid w, x)\).
The posterior distribution is approximated variationally as \(q(w; \theta)\) by optimizing the evidence lower bound \citep{blundell_weight_2015}:
\begin{align}
  \theta^* =
  \arg\min_\theta \Bigg\{
    -\expect _{w \sim q(w; \theta)} \log p(y \mid w, x)
    + D_{\text{KL}, w}(q(w;\theta) \mid p(w))
  \Bigg\}.
  \label{eq:elbo}
\end{align}
Our example uses the GeLU activation function and, like \citet[\S4.3]{wright_analytic_2024}, a single hidden layer; for other details, see App.~\ref{app:bayesian-networks}.

Monte Carlo Variational Inference (MCVI), a stochastic gradient method that approximates \(\expect_{w\sim q(w; \theta)}\) with Monte Carlo samples, is the gold standard for training variational Bayes neural networks.
In order to reduce the gradient variance and computational cost, \citet{wu_deterministic_2019,petersen_uncertainty_2024,wright_analytic_2024,rui_li_streamlining_2025} use  deterministic moment-matching for training (to compute gradients of evidence lower bound), and for inference 
(generate the predictive distribution).
While this setup evaluates the end-to-end learning process, it does not directly measure the accuracy of the underlying distributional approximation.
In particular, a sufficiently expressive variational model may compensate for systematic errors in distribution propagation, rendering the overall inference procedure robust to approximation bias.
However, strong end-to-end performance does not necessarily imply that the propagated distributions themselves are accurate.
\footnote{\,``Learning can still improve a bound on the log likelihood of the data even when the posterior distribution over hidden states is computed incorrectly''; see also the surrounding discussion in \citet[\S1.2]{frey_variational_1999}.
For a toy example, take \(Z \sim \mathcal N(0, 1)\); if the uncertainty propagation formula were \(aZ \sim \mathcal N(0, 2a)\), which is patently incorrect, the variational model could learn to approximate the target distribution \(\mathcal N(0, \sigma^2)\) using the parameter \(a = \sigma^2/2\).
If the approximations were accurate, one would expect the different inference methods to be indistinguishable from each other and from MCVI, as they are in \citet[Table 3]{wright_analytic_2024}.
}

Because our interest is in the correctness of distributional approximation, our experiments differ from the works cited in the following ways:
\begin{enumerate}
  \item We train the Bayesian neural network using MCVI with a large Monte Carlo batch size.
  After training has converged, the Monte Carlo predictions from this network are taken as ground truth mean and variance.
  \item We test the Bayesian neural network using Monte Carlo on the variational posterior, as well as six different techniques to propagate distributions through activation functions: the power series of \citet{wright_analytic_2024}, expanded to 1--5 terms, and our method.
  \item The figure of merit is the KL divergence (lower is better) from the pseudo-true Gaussian distribution (via Monte Carlo) to each deterministic approximation.
\end{enumerate}

We apply this method to eleven regression datasets from the UCI Machine Learning Repository.
The dataset list and full results are in Appendix~\ref{app:bayesian-networks};
we highlight the Parkinsons telemonitoring and concrete compressive strength datasets here.
As Fig.~\ref{fig:bayesian-network-telemonitoring} shows, the power series expansion of \citet{wright_analytic_2024} shows a ``dose response,'' i.e., it becomes more accurate as more terms are added.
But the most accurate approximation to the Monte Carlo predictive distribution is attained using exact moment matching, our method.
It is not exact because, for a Bayesian neural network with even a single hidden layer, 
Gaussian moment matching introduces additional error due to multiplication of random weights and layer inputs
\citep[eqq.~3--6]{goulet_tractable_2021}.

\begin{figure}
  \centering
  \begin{subfigure}{.49\linewidth}
    \centering
    \includegraphics[width=\linewidth]{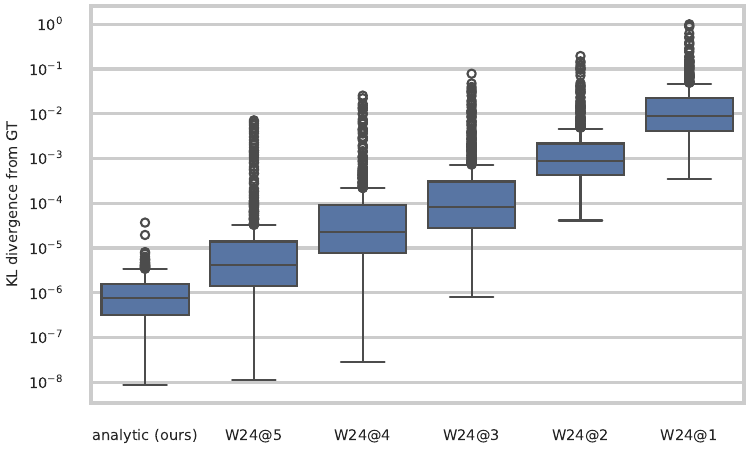}
    \caption{Parkinsons telemonitoring}
  \end{subfigure}
  \hfill
  \begin{subfigure}{.49\linewidth}
    \centering
    \includegraphics[width=\linewidth]{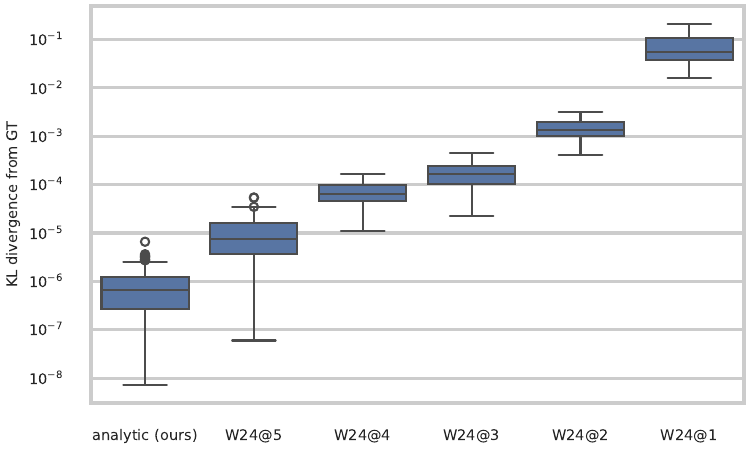}
    \caption{Concrete compressive strength}
  \end{subfigure}
  \caption{\label{fig:bayesian-network-telemonitoring} KL divergence between ground-truth predictive distribution (by Monte Carlo) and approximations for two UCI regression datasets. W24@$k$ means \citet{wright_analytic_2024} with $k$ terms in the series expansion.}
\end{figure}

\subsection{Stochastic activations}
\label{sec:stochastic-networks}
The 2024 Nobel Prize in Physics recognized the stochastic, biological inspiration for artificial neurons in machine learning \citep{davour_nobel_nodate}.
Even though biological motivation is no longer front of mind, as a curiosity we analyze the output distribution of a neural network whose activations are random processes modeled after a stochastic neuron:
\begin{align}
  \tilde \sigma(x, U) &= 2 \bm{1}_{U < \Phi(x)} - 1,
\end{align}
where at each artificial neuron, \(U \sim \operatorname{Uniform}(0,1)\) is an independent random variable.

In Appendix~\ref{app:stochastic-networks} we derive a moment-matching approximation to the distribution of the output of a stochastic neural network.
Applying this formula to 1 million samples from a stochastic version of the ``deep'' neural network of \S\ref{sec:random-networks} with a constant zero input results in a normal distribution with a good agreement (Fig.~\ref{fig:stochastic}).
Whether this line of inquiry deserves further methodological development is reserved for future work.

\section{Limitations}
A limitation of our moment matching derivation is that it does not apply to softmax (including attention), which appears to remain an open problem.
Because the numerical examples are scoped to presenting our theory and derivations, we do not test on large problems or benchmark against other methodologies for applied machine learning.

\section{Extensions}

\subsection{Residual connections across multiple layers}
It is possible to represent ResNets with multi-layer skips within the formalism \(\sigma(Ax+b)+Cx+d\), by using block matrices to pass the residual stream through each layer.
Consider an \(n\)-deep residual connection \(\tilde f(x)\), defined by
\begin{align*}
  \tilde f(x) &= x+f^n(x), &
  f^0(x) &= x, \\
  f^k(x) &= \sigma\del{A^k f^{k-1}(x)+b^k}, &
  k &\in [n].
\end{align*}
This can be represented by enlarging the hidden layers to \(\tilde f^k(x):=(f^k(x),x)\).
For any vector \(b_0\) satisfying \(\sigma(b_0)=0\) (e.g., the zero vector for sine or GeLU),
\begin{align*}
  \tilde f^1(x)
  &=
  \sigma\del{
    \begin{pmatrix} A^1 \\ 0 \end{pmatrix}x
    +\begin{pmatrix} b^1 \\ b_0 \end{pmatrix}}
  +\begin{pmatrix} 0 \\ I \end{pmatrix}x, \\
  \tilde f^k(x)
  &=
  \sigma\del{
    \begin{pmatrix} A^k & 0 \\ 0 & 0 \end{pmatrix}\tilde f^{k-1}(x)
    +\begin{pmatrix} b^k \\ b_0 \end{pmatrix}}
  +\begin{pmatrix} 0 & 0 \\ 0 & I \end{pmatrix}\tilde f^{k-1}(x), 
  & k \in [2\ldots n]\\
  \tilde f(x)
  &= \begin{pmatrix} I & I \end{pmatrix}\tilde f^n(x).
\end{align*}
In this way, it is possible to moment-match across multi-layer skips by tracking the cross-covariances between the activation and the residual stream.

\subsection{Kalman filtering}
In Kalman filtering, a state estimation algorithm for dynamic systems, it is important to propagate the belief state \(X_t\sim\mathcal N(\mu_t,\Sigma_t)\) through the recursion \(X_{t+1}=f(X_t)\) and the measurement law \(Y_t=h(X_t)\).
In the measurement update step \citep[Chapter~8]{sarkka_bayesian_2023}, a function \(h\) is given, and the task is to compute the covariance between \(X\) and \(h(X)\).
Moment matching results in an approximately Gaussian distribution for the pair \((X,h(X))\), whose covariance matrix's off-diagonal blocks contain the desired cross-covariance:
\begin{align*}
  \Cov\begin{pmatrix} X \\ h(X) \end{pmatrix}
  =
  \begin{pmatrix}
    \Cov X & \Cov(X,h(X)) \\
    \Cov(h(X),X) & \Cov h(X)
  \end{pmatrix}.
\end{align*}
Moment matching is a promising numerical scheme for joint state and measurement moments, and has been developed in concurrent work \citep{kuang_assumed_2026}.

\subsection{Misclassification under Gaussian perturbations}
One application of evaluating \(f(X)\) under Gaussian input is to study the robustness of the network \(f\) to random perturbations.
Suppose \(f\) gives the class scores of a classifier and \(i^*\) is the correct label.
A competing class \(i\ne i^*\) is misclassified with probability
\begin{align*}
  \Pr\del{f_i(X)>f_{i^*}(X)}
   = \expect\del{\operatorname{Heaviside}\del{f_i(X)-f_{i^*}(X)}},
\end{align*}
which can be solved using moment matching.

\section{Novelty and significance}


Our theoretical analysis shows that Gaussian moment matching is second-order accurate.
Our moment matching method is the first to achieve this guarantee.

Prior works on uncertainty propagation in neural networks have relied on approximations for the first and second moments under Gaussian inputs, including independence assumptions, approximate activation functions, and truncated series expansions.
In contrast, we derive exact, closed-form expressions for these moments for a range of widely used activation functions, enabling layer-wise moment propagation in deep networks without restrictive assumptions.
Our moment matching methods support general multivariate Gaussian input distributions and residual network architectures.
Empirically, we show that even beyond the guaranteed regime, the resulting method achieves orders-of-magnitude improvements in distributional accuracy over existing approaches.
We further demonstrate its effectiveness in uncertainty-aware inference on real data and in deterministic inference for variational Bayesian neural networks.
\clearpage


\section*{Reproducibility statement}
In the supplementary material, 
executable Python scripts are in \texttt{demo/}.
They reference libraries in \texttt{neural\_uncertainty\_propagation/} and generate output in \texttt{docs/manuscript/generated/}.
\texttt{deep\_sine\_convergence.py} generates Fig.~\ref{fig:deep-sine-convergence} and the results described in Appendix~\ref{app:deep-sine-convergence}.
\texttt{test\_case.py} generates the random-network tables and figures in the accompanying \emph{Supplementary Numerical Results}.
\texttt{test\_case\_summary.py} uses its CSV output to generate the KL summary boxplots.
To reproduce \S\ref{app:stochastic-networks}, run \texttt{stochastic.py}.
To reproduce \S\ref{app:bayesian-networks}, run \texttt{bayesian-network.py}.

\section*{Acknowledgements}
This work was supported by the National Science Foundation CAREER Program (Grant No.~2046292).


\clearpage
\bibliographystyle{plainnat}
\bibliography{nn-filtering.bib}


\appendix
\clearpage
\onecolumn
\section*{Supplementary Material}

\tableofcontents
\clearpage

\listoffigures
\clearpage

\section{Baselines}
\label{sec:baselines}
\subsection{The mean-field approximation \(Y_\mathrm{mfa}\)}
Following \citet{huber_bayesian_2020, wagner_kalman_2022,akgul_deterministic_2025}, we define the mean-field analytic approximation by assuming that neurons in the same hidden layer are independent:
\begin{align*}
  Y_\mathrm{mfa} &=  Y^\ell\\
  Y^k &= \mathcal N(\mu^k, \Sigma^k) & \forall k &\in \cbr{1 \ldots \ell} \\
  \mu^k &= \expect g(Y^{k-1}; A^k, b^k, C^k, d^k)
  \\
  \Sigma^k_{ij} &= \begin{cases}
    \sbr{\Cov g(Y^{k-1}; A^k, b^k, C^k, d^k)}_{ij}, & i=j
    \\
    0, &\text{else}
    \tag{mean-field assumption}
  \end{cases}
  \\
  Y^0 &= X
\end{align*}

\subsection{The linear approximation \(Y_\mathrm{lin}\)}
Following \citet{titensky_uncertainty_2018, nagel_kalman-bucy-informed_2022,petersen_uncertainty_2024, jungmann_analytical_2025}, we define the linear approximation as the asymptotically Normal output distribution (pursuant to the delta method) in the limit of small input variance:
\begin{align*}
  Y_\mathrm{lin} &= \mathcal N(f(\expect X), \nabla f(\expect X) \Cov X \nabla f(\expect X)^\intercal)
\end{align*}

\subsection{The unscented approximation(s) \(Y_\mathrm{u95}\) and \(Y_\mathrm{u02}\)}
An unscented transformation is a quadrature rule for approximating a probability measure on \(\mathbb{R}^n\) by \(2n + 1\) point masses whose locations (called sigma points) and weights satisfy certain first- and second-order moment matching conditions.
This technique was developed to improve upon linearization in nonlinear Kalman filtering \citep{julier_new_1995,julier_new_1997,julier_new_2000,julier_scaled_2002,wan_unscented_2000,julier_unscented_2004}.

There are (at least) two unscented transformations: a one-parameter family, which we refer to as Unscented'95; and a three-parameter family, which we refer to as Unscented'02.
\begin{itemize}
  \item Unscented'95 \citep{julier_new_1995,julier_new_1997,julier_new_2000}, used in \citet{astudillo_propagation_2011,abdelaziz_uncertainty_2015}, is a one-parameter family of unscented transformations, parameterized by a single shape hyperparameter \(\kappa\).
  \item Unscented'02 \citep{julier_scaled_2002,wan_unscented_2000} is a three-parameter family of unscented transformations, parameterized by three shape hyperparameters \((\alpha, \beta, \kappa)\).
  This version, with default values tuned to be drastically more localized in \(X\)-space than Unscented'95, is usually called ``the'' unscented transformation in current filtering research \citep{jiang_new_2025} and tooling \citep{ljung_unscentedkalmanfilter_2025}.
\end{itemize}

\section{Preliminaries: Gaussian integrals}
\label{app:gaussian-integrals}
\begin{definition}
  The bivariate normal CDF \(\Phi_2\) is defined by
  \begin{align*}
    \Phi_2(h, k; \rho)
    &= \probability\sbr{Z_1 \leq h, Z_2\leq k},
    \intertext{where}
    \begin{pmatrix}
      Z_1 \\ Z_2
    \end{pmatrix}
    &\sim \mathcal N\del{\begin{pmatrix} 0 \\ 0 \end{pmatrix}, \begin{pmatrix}
      1 & \rho \\ \rho & 1
    \end{pmatrix}}.
  \end{align*}
\end{definition}
\begin{remark}
  While we consider it a valid ``closed form'' atom, the bivariate normal CDF can be a difficult transcendental function
  to evaluate.
  Many software packages such as SciPy \citep{wagner_kalman_2022}
  implement the multivariate normal CDF by a (quasi-) Monte Carlo
  integration over \(\mathbb{R}^n\), which is too expensive for our purposes.
  Furthermore, we frequently use the expression \(\Phi(h, k; \rho) - \Phi(h, k; 0)\), which is
  vulnerable to cancellation error for extreme values of \(h\),
  \(k\), and \(\rho\).
  To avoid this, we compute the increment
  \(\Delta \Phi_2(h, k; \rho) := \Phi_2(h, k; \rho) - \Phi_2(h, k; 0)\)
  as the one-dimensional proper integral
  \begin{align*}
    \Delta \Phi_2(h, k; \rho)
    &= \int_0^\rho \partial_{\rho'} \Phi_2(h, k; \rho') \dif\rho',
  \end{align*}
  using
  \begin{align*}
    \partial_\rho \Phi_2(h, k; \rho)
    &= \frac{1}{2\pi \sqrt{1 - \rho^2}} e^{-\frac{1}{2(1 - \rho^2)}
    \del{h^2 + k^2 - 2 \rho h k}},
  \end{align*}
  a helpful identity found in \citet{drezner_computation_1990}.
  Substituting \(\rho' = \sin\theta\) with \(\theta = z \arcsin \rho\), the Jacobian
  \(\cos\theta \dif\theta\) cancels the factor \((1 - \rho'^2)^{-1/2}\), giving
  \begin{align*}
    \Delta \Phi_2(h, k; \rho)
    &= \int_0^{\arcsin \rho}
    \frac{1}{2\pi} \exp\sbr{-\frac{h^2 + k^2 - 2 \sin\theta \, h k}{2 \cos^2\theta}}
    \dif\theta
    \\
    &=  \frac{\arcsin \rho}{2\pi}
    \int_0^{1}
    \exp\sbr{-\frac{h^2 + k^2 - 2 \sin((\arcsin \rho) z) \, h k}{2 \cos^2 ((\arcsin \rho) z)}}
    \dif z.
  \end{align*}
  We apply 10-point Gauss--Legendre quadrature to this integrand, which is analytic
  on \([0, 1]\) for every \(\abs{\rho} \leq 1\).
\end{remark}

\begin{definition}
  The bivariate Normal density is
  \begin{align*}
    \phi_2(h, k; \rho)
    &= \dmd{}{2}{h}{}{k}{}\Phi_2(h, k; \rho).
  \end{align*}
\end{definition}
\begin{definition}
  The partial derivative of the joint CDF is
  \begin{align}
    \Phi_{2;1}(h, k; \rho) = \dpd{}{h}
    \Phi_2(h, k ; \rho)
  \end{align}
\end{definition}
\begin{lemma}
  The function \(\Phi_{2;1}\) satisfies
  \begin{align*}
    \Phi_{2;1}(h, k; \rho)
    &= \phi(h) \Phi\del{\frac{k - \rho h}{\sqrt{1 - \rho^2}}}.
  \end{align*}
\end{lemma}
\begin{proof}
  Letting \(Z_1, Z_2\) be standard normal with correlation \(\rho\), there is a probabilistic interpretation of \(\Phi_{2;1}\):
  \begin{align*}
    \Phi_{2;1}(h, k; \rho)
    &= \underbrace{f_{Z_1}(h)}_{\text{marginal density}}
    \underbrace{\probability\del{Z_2 \leq k \mid Z_1 = h}}_{\text{conditional cdf}}
    \\
    &= \phi(h) \Phi\del{\frac{k - \rho h}{\sqrt{1 - \rho^2}}}.
  \end{align*}
\end{proof}

\begin{lemma}[Multivariate Stein's lemma]
\label{lem:stein-multivariate}
  Let \((X_1,\ldots, X_n)\) be a multivariate normal random vector.
  Then
  \begin{align*}
    \expect (X_1-\expect X_1) f(X_1,\ldots, X_n)
    &=
    \sum_{i = 1}^n \Cov(X_1, X_i) \expect \sbr{
      \partial_{X_i} f(X_1,\ldots, X_n)
    }
  \end{align*}
\end{lemma}

\begin{lemma}[Stein simplification of \(L_\sigma\)]
  \label{lem:stein}
  Let \(L_\sigma\) be defined as in Definition \ref{def:moment-maps}, and suppose that \(\sigma\) is differentiable.
  Then
  \begin{align*}
    L_\sigma(\mu_1; \nu_{11}, \nu_{22}, \nu_{12}) &= \nu_{12}
    \expect \sigma'(Z_1),
    & Z_1 &\sim \mathcal N(\mu_1, \nu_{11}).
  \end{align*}
\end{lemma}
\begin{proof}
  Straightforward application of Lemma~\ref{lem:stein-multivariate}.
\end{proof}

\begin{lemma}[Univariate integrals]
\label{lem:gaussian-univariate}
  If \(Z \sim \mathcal N(\mu, \nu)\), then
  \begin{subequations}
  \begin{align}
    \expect \Phi(Z) &= \Phi\del{\frac{\mu}{\sqrt{1 + \nu}}}
    \label{eq:gaussian-Phi}
    \\
    \expect \phi(Z) &= \frac{1}{\sqrt{1 + \nu}}
    \phi\del{\frac{\mu}{\sqrt{1 + \nu}}}
    \label{eq:gaussian-phi}
    \\
    \expect Z \Phi(Z)
    &= \frac{\mu}{(1 + \nu)^{3/2}}
    \phi\del{\frac{\mu}{\sqrt{1 + \nu}}}
    \label{eq:gaussian-ZPhi}
  \end{align}
  \end{subequations}
\end{lemma}
\begin{proof}[Proof of \eqref{eq:gaussian-Phi}]
  Introducing an independent \(Z \sim \mathcal{N}(0, 1)\), we have
  \begin{align}
    \expect \Phi(Z)
    &= \expect \probability \sbr{{Z \leq X} \mid X}
    \\
    &= \probability\sbr{{Z \leq X}}
    \tag{by the law of total probability}
    \\
    &=  \probability \sbr{Z - X \leq 0}
  \end{align}
  We conclude by noting that the random variable \(Z - X\) has a
  Normal distribution with mean \(-\mu\) and variance \(1 + \nu\).
\end{proof}
\begin{proof}[Proof of \eqref{eq:gaussian-phi}]
  We use \(\phi = \Phi'\).
  \begin{align}
    \expect \phi(Z)
    &= \expect \left.\dod{}{t}\right|_{t = 0}
    \Phi(Z+t)
    \\
    &= \left.\dod{}{t}\right|_{t = 0}
    \expect \Phi(Z+t)
    \tag{dominated convergence theorem}
    \\
    &= \left.\dod{}{t}\right|_{t = 0}
    \Phi\del{\frac{\mu + t}{\sqrt{1 + \nu}}}
    \tag{by \eqref{eq:gaussian-Phi}}
  \end{align}
\end{proof}
\begin{proof}[Proof of \eqref{eq:gaussian-ZPhi}]
  We use the Gaussian ODE
  \begin{align}
    \phi'(x) + x \phi(x) &= 0.
  \end{align}
  Centering and applying Lemma~\ref{lem:stein-multivariate},
  \begin{align}
    \expect Z \phi(Z)
    &= \expect (Z - \mu) \phi(Z) + \mu \expect \phi(Z)
    \tag{centering}
    \\
    &= \nu\expect \phi'(Z) + \mu \expect \phi(Z)
    \tag{Stein's}
    \\
    &= -\nu \expect Z \Phi (Z) + \mu \expect \phi(Z)
    \tag{Gaussian ODE}
    \intertext{Collecting like terms and solving for (I),}
    \del{1 + \nu} \expect Z \phi (Z)
    &= \mu \expect \phi(Z)
    \\
    \expect Z \phi (Z)
    &= \frac{\mu}{1 + \nu} \expect \phi(Z)
    \\
    &= \frac{\mu}{(1 + \nu)^{3/2}} \phi\del{
      \frac{\mu}{\sqrt{1 + \nu}}
    }
    \tag{by \eqref{eq:gaussian-phi}}
  \end{align}
\end{proof}

\begin{lemma}[Bivariate integrals]
  \label{lem:gaussian-bivariate}
  Let
  \begin{align*}
    \begin{pmatrix}
      X_1 \\ X_2
    \end{pmatrix}
    \sim \mathcal{N}\del{
      \begin{pmatrix}
        \mu_1 \\ \mu_2
      \end{pmatrix},
      \begin{pmatrix}
        \nu_{11} & \nu_{12} \\ \nu_{12} & \nu_{22}
      \end{pmatrix}
    }.
  \end{align*}
  Then
  \begin{subequations}
  \begin{align}
    \expect \Phi(X_1) \Phi(X_2)
    &=
    \Phi_2\del{
      \frac{\mu_1}{\sqrt {1+ \nu_{11}}},
      \frac{\mu_2}{\sqrt {1+ \nu_{22}}};
      \frac{\nu_{12}}{\sqrt{(1 + \nu_{11})(1 + \nu_{22})}}
    }
    \label{eq:gaussian-Phi-Phi}\\
    \expect \phi(X_1) \Phi(X_2)
    &= 
    \frac{1}{\sqrt{1 + \nu_{11}}}
    \Phi_{2;1}
    \del{
      \frac{\mu_1}{\sqrt{1 + \nu_{11}}}
      , \frac{\mu_2}{\sqrt{1 + \nu_{22}}};
      \frac{\nu_{12}}{\sqrt{(1 + \nu_{11})(1 + \nu_{22})}}
    }
    \label{eq:gaussian-phi-Phi}
    \\
    \expect \phi(X_1) \phi(X_2)
    &=
    \frac{1}{\sqrt{(1 + \nu_{11})( 1 + \nu_{22})}}
    \phi_{2}\del{
      \frac{\mu_1}{\sqrt{1 + \nu_{11}}},
      \frac{\mu_2}{\sqrt{1 + \nu_{22}}};
      \frac{\nu_{12}}{\sqrt{(1 + \nu_{11})(1 + \nu_{22})}}
    }
    \label{eq:gaussian-phi-phi}
    \\
    \expect \phi'(X_1) \Phi(X_2)
    &= 
    \frac{-\mu_1}{1 + \nu_{11}} \expect \phi(X_1) \Phi(X_2)
    + \frac{-\nu_{12}}{1 + \nu_{11}} \expect \phi(X_1) \phi(X_2)
    \label{eq:gaussian-phi'-Phi}
  \end{align}
  \end{subequations}
\end{lemma}
\begin{proof}[Proof of \eqref{eq:gaussian-Phi-Phi}]
  Introduce independent
  \(Z_1, Z_2 \sim \mathcal N(0, 1)\).
  Then
  \begin{align}
    \expect \Phi(X_1) \Phi (X_2) 
    &= \expect \probability\sbr{Z_1 \leq X_1 \mid X_1}
    \probability\sbr{Z_2 \leq X_2 \mid X_2}
    \\
    &= \expect \probability\sbr{Z_1 \leq X_1, Z_2 \leq X_2 \mid X_1, X_2}
    \tag{independence}
    \\
    &= \probability\sbr{Z_1 \leq X_1, Z_2 \leq X_2}
    \tag{by the law of total probability}
    \\
    &= \probability\sbr{Z_1 - X_1 \leq 0, Z_2 - X_2 \leq 0}.
  \end{align}
  We conclude by using the fact that \((Z_1 - X_1, Z_2 - X_2)\) is
  jointly Normal with distribution
  \begin{align}
    \begin{pmatrix}
      Z_1 - X_1
      \\
      Z_2 - X_2
    \end{pmatrix}
    \sim
    \mathcal{N}\del{
      \begin{pmatrix}
        -\mu_1
        \\
        -\mu_2
      \end{pmatrix},
      \begin{pmatrix}
        1 + \nu_{11}
        &
        \nu_{12}
        \\
        \nu_{12}
        &
        1 + \nu_{22}
      \end{pmatrix}
    }.
  \end{align}
\end{proof}
\begin{proof}[Proof of \eqref{eq:gaussian-phi-Phi}]
  Using \(\phi = \Phi'\), we have
   \begin{align}
    \MoveEqLeft \expect \phi(X_1) \Phi(X_2)
    \notag
    \\
    &= 
    \expect \left. \dod{}{t}\right|_{t=0} \Phi(X_1 + t) \Phi(X_2)
    \\
    &= \left. \dod{}{t}\right|_{t=0}
    \expect \Phi(X_1 + t) \Phi(X_2)
    \tag{dominated convergence theorem}\\
    &= \left. \dod{}{t}\right|_{t=0}
    \probability \del{
      Z_1 - X_1 - t \leq 0, Z_2 - X_2 \leq 0
    }
    \tag{introducing \(Z_1, Z_2 \sim \mathcal N(0, 1)\)}
    \\
    &= \left. \dod{}{t}\right|_{t=0}
    \Phi_2
    \del{
      \frac{\mu_1 + t}{\sqrt{1 + \nu_{11}}}
      , \frac{\mu_2}{\sqrt{1 + \nu_{22}}};
      \frac{\nu_{12}}{\sqrt{(1 + \nu_{11})(1 + \nu_{22})}}
    }
    \\
    &= \frac{1}{\sqrt{1 + \nu_{11}}}
    \Phi_{2;1}
    \del{
      \frac{\mu_1}{\sqrt{1 + \nu_{11}}}
      , \frac{\mu_2}{\sqrt{1 + \nu_{22}}};
      \frac{\nu_{12}}{\sqrt{(1 + \nu_{11})(1 + \nu_{22})}}
    }
  \end{align}
\end{proof}
\begin{proof}[Proof of \eqref{eq:gaussian-phi-phi}]
Using \(\phi = \Phi'\) in both terms,
\begin{align}
  \expect \phi(X_1) \phi(X_2)
    &= 
    \expect \left.\dmd{}{2}{t}{}{s}{}\right|_{t=0, s=0}
    \Phi(X_1 + t) \Phi(X_2 + s)
    \\
    &= \left.\dmd{}{2}{t}{}{s}{}\right|_{t=0, s=0}
    \expect \Phi(X_1 + t) \Phi(X_2 + s)
    \tag{dominated convergence theorem}
    \\
    &= \left.\dmd{}{2}{t}{}{s}{}\right|_{t=0, s=0}
    \probability\del{
      Z_1 - X_1 - t \leq 0, Z_2 - X_2 - s \leq 0
    }
    \tag{introducing \(Z_1, Z_2 \sim \mathcal N(0, 1)\)}
    \\
    &= \left.\dmd{}{2}{t}{}{s}{}\right|_{t=0, s=0}
    \Phi_{2}\del{
      \frac{\mu_1 + t}{\sqrt{1 + \nu_{11}}},
      \frac{\mu_2 + s}{\sqrt{1 + \nu_{22}}};
      \frac{\nu_{12}}{\sqrt{(1 + \nu_{11})(1 + \nu_{22})}}
    }
    \\
    &=
    \frac{1}{\sqrt{1 + \nu_{11}} \, \sqrt{1 + \nu_{22}}}
    \phi_{2}\del{
      \frac{\mu_1}{\sqrt{1 + \nu_{11}}},
      \frac{\mu_2}{\sqrt{1 + \nu_{22}}};
      \frac{\nu_{12}}{\sqrt{(1 + \nu_{11})(1 + \nu_{22})}}
    }
    \tag{pdf is derivative of cdf}
\end{align}
\end{proof}
\begin{proof}[Proof of \eqref{eq:gaussian-phi'-Phi}]
  We use the Gaussian ODE
  \begin{align}
    \phi'(x) + x \phi(x) &= 0.
  \end{align}
  Using this fact,
  \begin{align}
    \expect \phi'(X_1) \Phi(X_2)
    &= -\expect X_1 \phi(X_1) \Phi(X_2)
    \\
    &= -\mu_1 \expect \phi(X_1) \Phi(X_2)
      - \expect (X_1 - \mu_1) \phi(X_1) \Phi(X_2)
    \tag{centering}
    \\
    &= -\mu_1 \expect \phi(X_1) \Phi(X_2)
      - \nu_{11} \underbrace{
        \expect \phi'(X_1) \Phi(X_2)
      }_{\text{same as LHS}}
      - \nu_{12} \expect \phi(X_1) \phi(X_2)
      \tag{Lemma~\ref{lem:stein-multivariate}}
    \\
    \intertext{Collecting like terms and solving for \(\expect \phi'(X_1) \Phi(X_2)\),,}
    \expect \phi'(X_1) \Phi(X_2)
    &= \frac{-\mu_1}{1 + \nu_{11}} \expect \phi(X_1) \Phi(X_2)
      + \frac{-\nu_{12}}{1 + \nu_{11}} \expect \phi(X_1) \phi(X_2).
  \end{align}
\end{proof}

\section{Derivation of uncertainty propagation formulas for probit activation}
\label{app:probit}
In this appendix, we derive the \(M_\sigma\), \(K_\sigma\), and
\(L_\sigma\) functions (Def.~\ref{def:moment-maps}) for the normal
CDF activation function
\begin{align}
  \sigma(x) &= \sqrt{\frac{2}{\pi}} \int_{0}^x e^{-\frac{1}{2} u^2}
  du = 2 \Phi(x) - 1, \quad \Phi(x) = \probability_{Z \sim \mathcal
  N(0, 1)}(Z \leq x)
  \label{eq:activation}
\end{align}

\begin{lemma}
\label{lem:mean-probit}
  Let \(M_\sigma\) be defined as in Def.~\ref{def:moment-maps}, with
  \(\sigma\) as in \eqref{eq:activation}.
  \begin{align*}
    M_\sigma(\mu; \nu) &= \sigma\del{\frac{\mu}{\sqrt{1 + \nu}}}
  \end{align*}
  \label{lem:mean}
\end{lemma}
\begin{proof}
  Let \(X \sim \mathcal N(\mu, \nu)\).
  \begin{align}
    \expect \sigma(X)
    &= -1+2\expect \Phi\del{X}
  \end{align}
  This is a direct application of Lemma~\ref{lem:gaussian-univariate}, \eqref{eq:gaussian-Phi}.
\end{proof}

\begin{lemma}
  Let \(K_\sigma\) be defined as in Def.~\ref{def:moment-maps}, with
  \(\sigma\) as in \eqref{eq:activation}.
  \begin{align*}
    K(\mu_1, \mu_2; \nu_{11}, \nu_{22}, \nu_{12}) &=
    4 \left.
    \Phi_2\del{
      \frac{\mu_1}{\sqrt{1 + \nu_{11}}},
      \frac{\mu_2}{\sqrt{1 + \nu_{22}}};
      \rho'
    }
    \right|^{\rho' = \frac{\nu_{12}}{\sqrt{(1 +
    \nu_{11})(1 + \nu_{22})}}}_{\rho' = 0},
  \end{align*}
  where \(\Phi_2\) is the bivariate normal CDF.
\end{lemma}
\begin{proof}
  The covariance may be expressed as
  \begin{align}
    \Cov (\sigma(X_1), \sigma(X_2))
    &= 4 \Cov (\Phi(X_1), \Phi(X_2))
    \\
    &= 4\expect \Phi(X_1) \Phi (X_2) - 4\expect \Phi(X_1) \expect \Phi(X_2)
  \end{align}
  Now apply Lemma~\ref{lem:gaussian-bivariate} to the first term.
\end{proof}

\begin{lemma}
  Let \(L_\sigma\) be defined as in Def.~\ref{def:moment-maps}, with
  \(\sigma\) as in \eqref{eq:activation}.
  Then
  \begin{align*}
    L_\sigma(\mu_1; \nu_{11}, \nu_{22}, \nu_{12}) &= 2
    \frac{\nu_{12}}{\sqrt{1 + \nu_{11}}}
    \phi\del{\frac{\mu_1}{\sqrt{1 + \nu_{11}}}}.
  \end{align*}
\end{lemma}
\begin{proof}
  Using Lemma \ref{lem:stein}, we have
  \begin{align}
    L_\sigma(\mu_1; \nu_{11}, \nu_{22}, \nu_{12})
    &= \nu_{12} \expect \sigma'(X_1)
    \\
    &= 2\nu_{12} \expect \phi(X_1)
  \end{align}
  where \(\phi = \Phi'\).
  By Lemma~\ref{lem:gaussian-univariate}, applied to \(X_1 \sim \mathcal N(\mu_1, \nu_{11})\),
  \begin{align}
    \expect \phi(X_1) = \frac{1}{\sqrt{1 + \nu_{11}}}
    \phi\del{\frac{\mu_1}{\sqrt{1 + \nu_{11}}}}.
  \end{align}
\end{proof}

\section{Derivation of uncertainty propagation formulas for GeLU activation}
\label{app:gelu}
In this appendix, we derive the \(M_\sigma\), \(K_\sigma\) and
\(L_\sigma\) functions (Def.~\ref{def:moment-maps}) for the GeLU activation function
\begin{align}
  \label{eq:activation-gelu}
  \sigma(x) &= x \Phi(x).
\end{align}



\begin{lemma}
\label{lem:gelu M}
  Let \(M_\sigma\) be defined as in Def.~\ref{def:moment-maps}, with
  \(\sigma\) as in \eqref{eq:activation-gelu}.
  Then
  \begin{align*}
    M_\sigma(\mu; \nu)
    &= \frac{\nu}{\sqrt{1 + \nu}} \phi\del{\frac{\mu}{\sqrt{1 + \nu}}}
    + \mu \Phi\del{\frac{\mu}{\sqrt{1 + \nu}}}.
  \end{align*}
\end{lemma}
\begin{proof}
  Let \(X \sim \mathcal N(\mu, \nu)\) and \(Z \sim \mathcal N(0, 1)\).
  Centering \(X\),
  \begin{align*}
    M_\sigma(\mu; \nu) 
    = \expect X \Phi(X)
    &= \expect (X - \mu) \Phi(X) + \mu \expect \Phi(X)
    \\
    &= \expect (X - \mu) \Phi(X) + \mu \Phi\del{\frac{\mu}{\sqrt{1 + \nu}}}
    \tag{using Lemma~\ref{lem:gaussian-univariate}, \eqref{eq:gaussian-Phi}}
    \\
    &= \nu\expect \phi(X) + \mu \Phi\del{\frac{\mu}{\sqrt{1 + \nu}}}
    \tag{using Lemma~\ref{lem:stein-multivariate}}
    \\
    &=
    \frac{\nu}{\sqrt{1 + \nu}}
    \phi\del{\frac{\mu}{\sqrt{1 + \nu}}}
    + \mu \Phi\del{\frac{\mu}{\sqrt{1 + \nu}}}
    \tag{using Lemma~\ref{lem:gaussian-univariate}, \eqref{eq:gaussian-phi}}
  \end{align*}
\end{proof}


\begin{lemma}
\label{lem:gelu K}
  Let \(K_\sigma\) be defined as in Def.~\ref{def:moment-maps}, with
  \(\sigma\) as in \eqref{eq:activation-gelu}.
  Then
  \begin{multline*}
    K_\sigma(\mu_1, \mu_2; \nu_{11}, \nu_{22}, \nu_{12})
    \\
    = 
    \del{
      \mu_1 \nu_{12} + \mu_{2}\nu_{11}  - \frac{\mu_{1} \nu_{12} \nu_{11}}{
        1 + \nu_{11}
      }
    }
    \frac{1}{\sqrt{1 + \nu_{11}}}
    \Phi_{2;1}
    \del{
      \frac{\mu_1}{\sqrt{1 + \nu_{11}}}
      , \frac{\mu_2}{\sqrt{1 + \nu_{22}}};
      \frac{\nu_{12}}{\sqrt{(1 + \nu_{11})(1 + \nu_{22})}}
    }
    \\
    + \del{
      \mu_2 \nu_{12} + \mu_{1}\nu_{22} - \frac{\mu_{2} \nu_{12} \nu_{22}}{
        1 + \nu_{22}
      }
    }
    \frac{1}{\sqrt{1 + \nu_{22}}}
    \Phi_{2;1}
    \del{
      \frac{\mu_2}{\sqrt{1 + \nu_{22}}}
      , \frac{\mu_1}{\sqrt{1 + \nu_{11}}};
      \frac{\nu_{12}}{\sqrt{(1 + \nu_{11})(1 + \nu_{22})}}
    }
    \\
    + 
    \frac{
      \nu_{11} \nu_{22} + \nu_{12}^2
      \del{
        1 - \frac{\nu_{11}}{1 + \nu_{11}}
        - \frac{\nu_{22}}{1 + \nu_{22}}
      }
    }{\sqrt{(1 + \nu_{11})(1 + \nu_{22})}}
    \phi_{2}\del{
      \frac{\mu_1}{\sqrt{1 + \nu_{11}}},
      \frac{\mu_2}{\sqrt{1 + \nu_{22}}};
      \frac{\nu_{12}}{\sqrt{(1 + \nu_{11})(1 + \nu_{22})}}
    }
    \\
    + \del{\mu_1\mu_2 + \nu_{12}}
    \Phi_2\del{
      \frac{\mu_1}{\sqrt{1 + \nu_{11}}},
      \frac{\mu_2}{\sqrt{1 + \nu_{22}}};
      \frac{\nu_{12}}{\sqrt{(1 + \nu_{11})(1 + \nu_{22})}}
    }
    \\
    -M_\sigma(\mu_1; \nu_{11}) M_\sigma(\mu_2; \nu_{22})
  \end{multline*}
\end{lemma}
\begin{proof}
  Let
  \begin{align}
  \begin{pmatrix} X_1 \\ X_2 \end{pmatrix}
   \sim \mathcal N\left(\begin{pmatrix}
    \mu_1
    \\
    \mu_2
  \end{pmatrix},
  \begin{pmatrix}
    \nu_{11}
    &
    \nu_{12}
    \\
    \nu_{12}
    &
    \nu_{22}
  \end{pmatrix}\right).
  \end{align}
  Our strategy is to use the formula
  \begin{align*}
    \Cov (\sigma(X_1), \sigma(X_2))
    = \expect \sigma(X_1) \sigma(X_2) - \expect \sigma(X_1) \expect \sigma(X_2)
  \end{align*}
  and use Lemma~\ref{lem:gelu M} for the second term.
  For the cross term, we first center the random variables:
  \begin{multline}
    \expect X_1 \Phi(X_1) X_2 \Phi(X_2)
    = \underbrace{
      \expect (X_1 - \mu_1) \Phi(X_1) (X_2 - \mu_2) \Phi(X_2)
    }_{\text{(I)}}
    \\
    +
    \underbrace{
      \expect \mu_1 \Phi(X_1) (X_2 - \mu_2) \Phi(X_2)
    }_{\text{(II)}}
    \\
    +
    \underbrace{
      \expect (X_1 - \mu_1)\Phi(X_1) \mu_2  \Phi(X_2)
    }_{\text{(III)}}
    \\
    +
    \underbrace{
      \expect \mu_1 \Phi(X_1) \mu_2 \Phi(X_2)
    }_{\text{(IV)}}
    \label{eq:gelu K 1}
  \end{multline}
  For term (I) of \eqref{eq:gelu K 1},
  \begin{align}
    \text{(I)} &= \expect (X_1 - \mu_1) \Phi(X_1) (X_2 - \mu_2) \Phi(X_2)
    \\
    &= \sum_{i = 1}^2 \Cov(X_1, X_i)
    \expect \partial_{X_i} \del{\Phi(X_1) (X_2 - \mu_2) \Phi(X_2)}
    \tag{Lemma~\ref{lem:stein-multivariate}}
    \\
    &= 
    \underbrace{
      \nu_{11} \expect \phi(X_1) (X_2 - \mu_2) \Phi(X_2)
    }_{\text{(I.a)}}
    + \nu_{12} \expect 
      \Phi(X_1) \Phi(X_2)
      + \underbrace{
        \nu_{12}
         \expect \Phi(X_1) (X_2 - \mu_2) \phi(X_2)
      }_{\text{(I.b)}}
      \label{eq:gelu K 2}
  \end{align}
  We apply Lemma~\ref{lem:stein-multivariate} to (I.a) and (I.b) to get
  \begin{align}
    \text{(I.a)} + \text{(I.b)}
    &= 
    \nu_{11} \expect \phi(X_1) (X_2 - \mu_2) \Phi(X_2)
    + \nu_{12} \expect \Phi(X_1) (X_2 - \mu_2) \phi(X_2)
    \\
    \begin{split}
      &= \nu_{11} \expect \sbr{
        \nu_{12} \phi'(X_1) \Phi(X_2) + \nu_{22} \phi(X_1) \phi(X_2)
      }
      \\
      &\quad + \nu_{12} \expect \sbr{
        \nu_{12}\phi(X_1) \phi(X_2) + \nu_{22} \Phi(X_1) \phi'(X_2)
      }
    \end{split}
    \\
    \begin{split}
      &= \nu_{12} \nu_{11} \underbrace{
        \expect \phi'(X_1) \Phi(X_2)
      }_{\text{(I.c)}}
      + \nu_{12} \nu_{22} \underbrace{
        \expect \Phi(X_1) \phi'(X_2)
      }_{\text{(I.d)}}
      \\&\quad+\del{
        \nu_{11} \nu_{22} + \nu_{12}^2
      }
      \expect \phi(X_1) \phi(X_2)
    \end{split}
    \label{eq:gelu K 3}
  \end{align}
  By Lemma~\ref{lem:gaussian-bivariate}, \eqref{eq:gaussian-phi'-Phi}, we have
  \begin{align*}
    \text{(I.c)} &= \frac{-\mu_1}{1 + \nu_{11}} \expect \phi(X_1) \Phi(X_2)
      + \frac{-\nu_{12}}{1 + \nu_{11}} \expect \phi(X_1) \phi(X_2)
      \\
    \text{(I.d)}
    &= \frac{-\mu_2}{1 + \nu_{22}} \expect \Phi(X_1) \phi(X_2)
      + \frac{-\nu_{12}}{1 + \nu_{22}} \expect \phi(X_1) \phi(X_2)
  \end{align*}
  Combining these last two equations with \eqref{eq:gelu K 2} and \eqref{eq:gelu K 3}, (I) becomes
  \begin{multline}
    \text{(I)}
    = \nu_{12} \nu_{11} \del{
        \frac{-\mu_1}{1 + \nu_{11}} \expect \phi(X_1) \Phi(X_2)
        + \frac{-\nu_{12}}{1 + \nu_{11}} \expect \phi(X_1) \phi(X_2)
      }
      \\
      + \nu_{12} \nu_{22} \del{
        \frac{-\mu_2}{1 + \nu_{22}} \expect \Phi(X_1) \phi(X_2)
      + \frac{-\nu_{12}}{1 + \nu_{22}} \expect \phi(X_1) \phi(X_2)
      }
      \\
      +\del{
        \nu_{11} \nu_{22} + \nu_{12}^2
      }
      \expect \phi(X_1) \phi(X_2)
      \\
      + \nu_{12} \expect 
      \Phi(X_1) \Phi(X_2)
  \end{multline}
  which simplifies to
  \begin{multline}
    \text{(I)}
    = 
    \frac{-\mu_{1}\nu_{12} \nu_{11}}{
      1 + \nu_{11}
    }
    \expect \phi(X_1) \Phi(X_2)
    + \frac{-\mu_2 \nu_{12} \nu_{22}}{
      1 + \nu_{22}
    }
    \expect \Phi(X_1) \phi(X_2)
    \\
    + \sbr{
      \nu_{11} \nu_{22} + \nu_{12}^2
      \del{
        1 - \frac{\nu_{11}}{1 + \nu_{11}}
        - \frac{\nu_{22}}{1 + \nu_{22}}
      }
    } \expect \phi(X_1) \phi(X_2)
    \\
    + \nu_{12} \expect \Phi(X_1) \Phi(X_2)
  \end{multline}

  For terms (II-III) of \eqref{eq:gelu K 1}, we apply Lemma~\ref{lem:stein-multivariate} to get
  \begin{align*}
    \text{(II)} = \expect \mu_1 \Phi(X_1) (X_2 - \mu_2) \Phi(X_2)
    &= 
    \mu_1 \del{
      \nu_{12} \expect \phi(X_1) \Phi(X_2)
      +
      \nu_{22} \expect \Phi(X_1) \phi(X_2)
    }
    \\
    \text{(III)} = \expect (X_1 - \mu_1)\Phi(X_1) \mu_2  \Phi(X_2)
    &= 
    \mu_2 \del{
      \nu_{11} \expect \phi(X_1) \Phi(X_2)
      +
      \nu_{12} \expect \Phi(X_1) \phi(X_2)
    }
    \\
  \end{align*}
  Now \eqref{eq:gelu K 1} becomes:
  \begin{multline}
    \expect X_1 \Phi(X_1) X_2 \Phi(X_2)
    \\
    = 
    \del{
      \mu_1 \nu_{12} + \mu_{2}\nu_{11} - \frac{\mu_{1} \nu_{12} \nu_{11}}{
        1 + \nu_{11}
      }
    }
    \underbrace{
      \expect \phi(X_1) \Phi(X_2)
    }_{\text{(V.a)}}
    \\
    + \del{
      \mu_2 \nu_{12} + \mu_{1}\nu_{22} - \frac{\mu_{2} \nu_{12} \nu_{22}}{
        1 + \nu_{22}
      }
    }
    \underbrace{
      \expect \Phi(X_1) \phi(X_2)
    }_{\text{(V.b)}}
    \\
    + \sbr{
      \nu_{11} \nu_{22} + \nu_{12}^2
      \del{
        1 - \frac{\nu_{11}}{1 + \nu_{11}}
        - \frac{\nu_{22}}{1 + \nu_{22}}
      }
    }
    \underbrace{
      \expect \phi(X_1) \phi(X_2)
    }_{\text{(V.c)}}
    \\
    + \del{\mu_1\mu_2 + \nu_{12}} \underbrace{
      \expect \Phi(X_1) \Phi(X_2)
    }_{\text{(V.d)}}
    \label{eq:gelu K 5}
  \end{multline}
  By Lemma~\ref{lem:gaussian-bivariate}, \eqref{eq:gaussian-phi-Phi},
  \begin{align}
    \text{(V.a)}
    &=
    \frac{1}{\sqrt{1 + \nu_{11}}}
    \Phi_{2;1}
    \del{
      \frac{\mu_1}{\sqrt{1 + \nu_{11}}}
      , \frac{\mu_2}{\sqrt{1 + \nu_{22}}};
      \frac{\nu_{12}}{\sqrt{(1 + \nu_{11})(1 + \nu_{22})}}
    }
    \label{eq:gelu K 4}
    \intertext{and}
    \text{(V.b)} &= \frac{1}{\sqrt{1 + \nu_{22}}}
    \Phi_{2;1}
    \del{
      \frac{\mu_2}{\sqrt{1 + \nu_{22}}}
      , \frac{\mu_1}{\sqrt{1 + \nu_{11}}};
      \frac{\nu_{12}}{\sqrt{(1 + \nu_{11})(1 + \nu_{22})}}
    }.
  \end{align}
  By Lemma~\ref{lem:gaussian-bivariate}, \eqref{eq:gaussian-phi-phi},
  \begin{align}
    \text{(V.c)}
    &=
    \frac{1}{\sqrt{1 + \nu_{11}} \, \sqrt{1 + \nu_{22}}}
    \phi_{2}\del{
      \frac{\mu_1}{\sqrt{1 + \nu_{11}}},
      \frac{\mu_2}{\sqrt{1 + \nu_{22}}};
      \frac{\nu_{12}}{\sqrt{(1 + \nu_{11})(1 + \nu_{22})}}
    }.
  \end{align}
  By Lemma~\ref{lem:gaussian-bivariate}, \eqref{eq:gaussian-Phi-Phi},
  \begin{align}
    \text{(V.d)}
    &= \Phi_2\del{
      \frac{\mu_1}{\sqrt{1 + \nu_{11}}},
      \frac{\mu_2}{\sqrt{1 + \nu_{22}}};
      \frac{\nu_{12}}{\sqrt{(1 + \nu_{11})(1 + \nu_{22})}}
    }
  \end{align}
  
  The final form of \eqref{eq:gelu K 5} is
  \begin{multline}
    \expect X_1 \Phi(X_1) X_2 \Phi(X_2)
    \\
    = 
    \del{
      \mu_1 \nu_{12} + \mu_{2}\nu_{11}  - \frac{\mu_{1} \nu_{12} \nu_{11}}{
        1 + \nu_{11}
      }
    }
    \frac{1}{\sqrt{1 + \nu_{11}}}
    \Phi_{2;1}
    \del{
      \frac{\mu_1}{\sqrt{1 + \nu_{11}}}
      , \frac{\mu_2}{\sqrt{1 + \nu_{22}}};
      \frac{\nu_{12}}{\sqrt{(1 + \nu_{11})(1 + \nu_{22})}}
    }
    \\
    + \del{
      \mu_2 \nu_{12} + \mu_{1}\nu_{22} - \frac{\mu_{2} \nu_{12} \nu_{22}}{
        1 + \nu_{22}
      }
    }
    \frac{1}{\sqrt{1 + \nu_{22}}}
    \Phi_{2;1}
    \del{
      \frac{\mu_2}{\sqrt{1 + \nu_{22}}}
      , \frac{\mu_1}{\sqrt{1 + \nu_{11}}};
      \frac{\nu_{12}}{\sqrt{(1 + \nu_{11})(1 + \nu_{22})}}
    }
    \\
    + 
    \frac{
      \nu_{11} \nu_{22} + \nu_{12}^2
      \del{
        1 - \frac{\nu_{11}}{1 + \nu_{11}}
        - \frac{\nu_{22}}{1 + \nu_{22}}
      }
    }{\sqrt{(1 + \nu_{11})(1 + \nu_{22})}}
    \phi_{2}\del{
      \frac{\mu_1}{\sqrt{1 + \nu_{11}}},
      \frac{\mu_2}{\sqrt{1 + \nu_{22}}};
      \frac{\nu_{12}}{\sqrt{(1 + \nu_{11})(1 + \nu_{22})}}
    }
    \\
    + \del{\mu_1\mu_2 + \nu_{12}}
    \Phi_2\del{
      \frac{\mu_1}{\sqrt{1 + \nu_{11}}},
      \frac{\mu_2}{\sqrt{1 + \nu_{22}}};
      \frac{\nu_{12}}{\sqrt{(1 + \nu_{11})(1 + \nu_{22})}}
    }
    \label{eq:gelu K 6}
  \end{multline}
  The conclusion follows from subtracting \(M(\mu_1; \nu_{11}) M(\mu_2; \nu_{22})\).
\end{proof}



\begin{lemma}
\label{lem:gelu L}
  Let \(L_\sigma\) be defined as in Def.~\ref{def:moment-maps}, with
  \(\sigma\) as in \eqref{eq:activation-gelu}.
  Then
  \begin{align}
    L_\sigma(\mu_1; \nu_{11}, \nu_{22}, \nu_{12})
    &= \nu_{12}
    \frac{\mu_1}{(1 + \nu_{11})^{3/2}} \phi\del{
      \frac{\mu_1}{\sqrt{1 + \nu_{11}}}
    }
    +
    \nu_{12}
    \Phi\del{
      \frac{\mu_1}{\sqrt{1 + \nu_{11}}}
    }.
  \end{align}
\end{lemma}
\begin{proof}
  By Lemma~\ref{lem:stein},
  \begin{align*}
    L_\sigma(\mu_1; \nu_{11}, \nu_{22}, \nu_{12}) &= \nu_{12}
    \expect \sigma'(X_1),
    & X_1 &\sim \mathcal N(\mu_1, \nu_{11})
    \\
    &= \nu_{12}
    \underbrace{\expect X_1 \phi(X_1)}_{\text{(I)}}
    + \nu_{12} \underbrace{\expect \Phi(X_1)}_{\text{(II)}}
    \tag{product rule}
  \end{align*}
  By Lemma~\ref{lem:gaussian-univariate}, \eqref{eq:gaussian-ZPhi}
  \begin{align}
    \text{(I)}    
    &= \frac{\mu_1}{(1 + \nu_{11})^{3/2}} \phi\del{
      \frac{\mu_1}{\sqrt{1 + \nu_{11}}}
    }.
  \end{align}
  Term (II) is given by
  \begin{align}
    \text{(II)} = \expect \Phi(X_1)
    &= \Phi\del{
      \frac{\mu_1}{\sqrt{1 + \nu_{11}}}
    }
  \end{align}
  Applying terms (I)--(II),
  \begin{align}
    L_\sigma(\mu_1; \nu_{11}, \nu_{22}, \nu_{12})
    &= \nu_{12}
    \frac{\mu_1}{(1 + \nu_{11})^{3/2}} \phi\del{
      \frac{\mu_1}{\sqrt{1 + \nu_{11}}}
    }
    +
    \nu_{12}
    \Phi\del{
      \frac{\mu_1}{\sqrt{1 + \nu_{11}}}
    }
  \end{align}

\end{proof}

\section{Derivation of uncertainty propagation formulas for ReLU activation}
\label{app:relu}
In this appendix, we derive the uncertainty propagation formulas for ReLU activation 
\begin{align}
  \sigma(x)
  &= \max(0, x)
  \label{eq:activation-relu}
\end{align}
using the identity
\begin{align}
  \operatorname{ReLU}(x)
  &= \lim_{\alpha \to \infty} \alpha^{-1} \operatorname{GeLU}(\alpha x).
\end{align}
See \citet{muthen_moments_1990} for another way to derive \(M\) and \(K\).

\begin{lemma}
  \label{lem:relu M}
  Let \(M_\sigma\) be defined as in Def.~\ref{def:moment-maps}, with
  \(\sigma\) as in \eqref{eq:activation-relu}.
  Then
  \begin{align*}
    M_\sigma(\mu; \nu)
    &= \sqrt{\nu} \phi\del{\frac{\mu}{\sqrt{\nu}}}
    + \mu \Phi\del{\frac{\mu}{\sqrt{\nu}}}.
  \end{align*}
\end{lemma}
\begin{proof}
  By the dominated convergence theorem,
  \begin{align}
    M_\sigma(\mu; \nu)
    &= \lim_{\alpha \to \infty} \alpha^{-1} \expect M_{\operatorname{GeLU}}(\alpha \mu; \alpha^2 \nu)
    \\
    &= \lim_{\alpha \to \infty} \alpha^{-1}
    \sbr{
      \frac{\alpha^2 \nu}{\sqrt{1 + \alpha^2 \nu}} \phi\del{\frac{\alpha \mu}{\sqrt{1 + \alpha^2 \nu}}}
      + \mu \Phi\del{\frac{\alpha \mu}{\sqrt{1 + \alpha^2 \nu}}}
    }
    \tag{by Lemma~\ref{lem:gelu M}}
    \\
    &= 
    \sqrt{\nu} \phi\del{\frac{\mu}{\sqrt{\nu}}}
    + \mu \Phi\del{\frac{\mu}{\sqrt{\nu}}}
  \end{align}
\end{proof}

\begin{lemma}
\label{lem:relu K}
  Let \(K_\sigma\) be defined as in Def.~\ref{def:moment-maps}, with
  \(\sigma\) as in \eqref{eq:activation-relu}.
  Then
  \begin{gather*}
  \begin{split}
    K_\sigma(\mu_{1}, \mu_{2}; \nu_{11}, \nu_{22}, \nu_{12})
    &=
    \mu_2 \sqrt{\nu_{11}}
    \Phi_{2;1}
  \del{
    \frac{\mu_1}{\sqrt{\nu_{11}}}
    , \frac{\mu_2}{\sqrt{\nu_{22}}};
    \frac{\nu_{12}}{\sqrt{\nu_{11}\nu_{22}}}
  }
  \\
  &\quad +
  \mu_1 \sqrt{\nu_{22}}
  \Phi_{2;1}
  \del{
    \frac{\mu_2}{\sqrt{\nu_{22}}}
    , \frac{\mu_1}{\sqrt{\nu_{11}}};
    \frac{\nu_{12}}{\sqrt{\nu_{11}\nu_{22}}}
  }
  \\
  &\quad + 
  \del{
    \sqrt{\nu_{11}\nu_{22}}
    - \frac{\nu_{12}^2}{\sqrt{\nu_{11}\nu_{22}}}
  }
  \phi_{2}\del{
    \frac{\mu_1}{\sqrt{\nu_{11}}},
    \frac{\mu_2}{\sqrt{\nu_{22}}};
    \frac{\nu_{12}}{\sqrt{\nu_{11}\nu_{22}}}
  }
  \\
  &\quad + \del{ \mu_1\mu_2 +  \nu_{12}}
  \Phi_2\del{
    \frac{\mu_1}{\sqrt{\nu_{11}}},
    \frac{\mu_2}{\sqrt{\nu_{22}}};
    \frac{\nu_{12}}{\sqrt{\nu_{11}\nu_{22}}}
  }
  \\
  &\quad - M_\sigma(\mu_1; \nu_{11}) M_\sigma(\mu_2; \nu_{22})
  \end{split}
  \end{gather*}
\end{lemma}
\begin{proof}
By the dominated convergence theorem,
\begin{align}
  \expect \sigma(X_1) \sigma(X_2)
  &= \expect \lim_{\alpha \to \infty} \alpha^{-2} \operatorname{GeLU}(\alpha X_1) \operatorname{GeLU}(\alpha X_2)
  \\
  &= \lim_{\alpha \to \infty} \alpha^{-2} \expect  \operatorname{GeLU}(\alpha X_1) \operatorname{GeLU}(\alpha X_2)
\end{align}
To compute \(\alpha^{-2} \expect  \operatorname{GeLU}(\alpha X_1) \operatorname{GeLU}(\alpha X_2)\), we use \eqref{eq:gelu K 6} (Proof of Lemma~\ref{lem:gelu K}) while scaling \(\mu\)s by \(\alpha\) and \(\nu\)s by \(\alpha^2\).
\begin{align}
\begin{split}
    \MoveEqLeft\alpha^{-2} \expect  \operatorname{GeLU}(\alpha X_1) \operatorname{GeLU}(\alpha X_2)
    \\
    &= 
    \frac{
      \alpha \mu_1 \nu_{12} +\alpha \mu_{2}\nu_{11}  - \alpha^3 \frac{\mu_{1} \nu_{12} \nu_{11}}{
        1 + \alpha^2 \nu_{11}
      }
    }
    {\sqrt{1 + \alpha^2 \nu_{11}}}
    \\
     &\qquad \Phi_{2;1}
    \del{
      \frac{\alpha \mu_1}{\sqrt{1 + \alpha^2 \nu_{11}}}
      , \frac{\alpha \mu_2}{\sqrt{1 + \alpha^2 \nu_{22}}};
      \frac{\alpha^2 \nu_{12}}{\sqrt{(1 + \alpha^2 \nu_{11})(1 + \alpha^2 \nu_{22})}}
    }
    \\
    &\quad + \frac{
      \alpha \mu_2 \nu_{12} + \alpha \mu_{1}\nu_{22} - \alpha^3 \frac{\mu_{2} \nu_{12} \nu_{22}}{
        1 + \alpha^2 \nu_{22}
      }
    }{\sqrt{1 + \alpha^2 \nu_{22}}}
    \\
    &\qquad \Phi_{2;1}
    \del{
      \frac{\alpha \mu_2}{\sqrt{1 + \alpha^2 \nu_{22}}}
      , \frac{\alpha \mu_1}{\sqrt{1 + \alpha^2 \nu_{11}}};
      \frac{\alpha^2 \nu_{12}}{\sqrt{(1 + \alpha^2 \nu_{11})(1 + \alpha^2 \nu_{22})}}
    }
    \\
    &\quad + 
    \frac{
      \alpha^2 \nu_{11} \nu_{22} + \alpha^2\nu_{12}^2
      \del{
        1 - \frac{\alpha^2 \nu_{11}}{1 + \alpha^2 \nu_{11}}
        - \frac{\alpha^2 \nu_{22}}{1 + \alpha^2 \nu_{22}}
      }
    }{\sqrt{(1 + \alpha^2 \nu_{11})(1 + \alpha^2 \nu_{22})}}
    \\
    &\qquad \phi_{2}\del{
      \frac{\alpha \mu_1}{\sqrt{1 + \alpha^2 \nu_{11}}},
      \frac{\alpha \mu_2}{\sqrt{1 + \alpha^2 \nu_{22}}};
      \frac{\alpha^2 \nu_{12}}{\sqrt{(1 + \alpha^2 \nu_{11})(1 + \alpha^2 \nu_{22})}}
    }
    \\
    &\quad + \del{ \mu_1\mu_2 +  \nu_{12}}
    \Phi_2\del{
      \frac{\alpha \mu_1}{\sqrt{1 + \alpha^2 \nu_{11}}},
      \frac{\alpha \mu_2}{\sqrt{1 + \alpha^2 \nu_{22}}};
      \frac{\alpha^2 \nu_{12}}{\sqrt{(1 + \alpha^2 \nu_{11})(1 + \alpha^2 \nu_{22})}}
    }
\end{split}
\end{align}
  
Taking \(\alpha \to \infty\),
\begin{align}
\begin{split}
   \expect \sigma(X_1) \sigma(X_2)
  &=\lim_{\alpha \to \infty}\alpha^{-2} \expect  \operatorname{GeLU}(\alpha X_1) \operatorname{GeLU}(\alpha X_2)
  \\
  &=
  \mu_2 \sqrt{\nu_{11}}
    \Phi_{2;1}
  \del{
    \frac{\mu_1}{\sqrt{\nu_{11}}}
    , \frac{\mu_2}{\sqrt{\nu_{22}}};
    \frac{\nu_{12}}{\sqrt{\nu_{11}\nu_{22}}}
  }
  \\
  &\quad +
  \mu_1 \sqrt{\nu_{22}}
  \Phi_{2;1}
  \del{
    \frac{\mu_2}{\sqrt{\nu_{22}}}
    , \frac{\mu_1}{\sqrt{\nu_{11}}};
    \frac{\nu_{12}}{\sqrt{\nu_{11}\nu_{22}}}
  }
  \\
  &\quad + 
  \del{
    \sqrt{\nu_{11}\nu_{22}}
    - \frac{\nu_{12}^2}{\sqrt{\nu_{11}\nu_{22}}}
  }
  \phi_{2}\del{
    \frac{\mu_1}{\sqrt{\nu_{11}}},
    \frac{\mu_2}{\sqrt{\nu_{22}}};
    \frac{\nu_{12}}{\sqrt{\nu_{11}\nu_{22}}}
  }
  \\
  &\quad + \del{ \mu_1\mu_2 +  \nu_{12}}
  \Phi_2\del{
    \frac{\mu_1}{\sqrt{\nu_{11}}},
    \frac{\mu_2}{\sqrt{\nu_{22}}};
    \frac{\nu_{12}}{\sqrt{\nu_{11}\nu_{22}}}
  }
  \end{split}
\end{align}
The conclusion follows after subtracting \(\expect \sigma(X_1) \expect \sigma(X_2)\).
\end{proof}

\begin{lemma}
  \label{lem:relu L}
  Let \(L_\sigma\) be defined as in Def.~\ref{def:moment-maps}, with
  \(\sigma\) as in \eqref{eq:activation-relu}.
  Then
  \begin{align*}
    L_\sigma(\mu_1; \nu_{11}, \nu_{22}, \nu_{12})
    &= 
    \nu_{12}
    \Phi\del{
      \frac{\mu_1}{\sqrt{\nu_{11}}}
    }.
  \end{align*}
\end{lemma}
\begin{proof}
  By the dominated convergence theorem,
  \begin{align}
    \MoveEqLeft
    L_\sigma(\mu_1; \nu_{11}, \nu_{22}, \nu_{12})
    \notag\\
    &= \lim_{\alpha \to \infty} \alpha^{-2} L_{\operatorname{GeLU}}(\alpha \mu_1; \alpha^2 \nu_{11}, \alpha^2 \nu_{22}, \alpha^2 \nu_{12}) 
    \\
    &=
    \lim_{\alpha \to \infty} \alpha^{-2} 
    \sbr{
      \alpha^2 \nu_{12}
      \frac{\alpha \mu_1}{(1 + \alpha^2 \nu_{11})^{3/2}} \phi\del{
        \frac{\alpha \mu_1}{\sqrt{1 + \alpha^2 \nu_{11}}}
      }
      +
      \alpha^2 \nu_{12}
      \Phi\del{
        \frac{\alpha \mu_1}{\sqrt{1 + \alpha^2 \nu_{11}}}
      }
    }
    \tag{by Lemma~\ref{lem:gelu L}}
    \\
    &=
    \nu_{12}
    \Phi\del{
      \frac{\mu_1}{\sqrt{\nu_{11}}}
    } 
  \end{align}
\end{proof}

\section{Derivation of uncertainty propagation formulas for Heaviside activation}
\label{app:heaviside}
In this appendix, we state the \(M_\sigma\), \(K_\sigma\) and
\(L_\sigma\) functions (Def.~\ref{def:moment-maps}) for the Heaviside activation function
\begin{align}
  \label{eq:activation-heaviside}
  \sigma(x) &= \bm{1}_{x > 0}
  \\
  &= \lim_{\alpha \to \infty} \del{1/2  + 1/2 \hat\sigma(\alpha x)}
\end{align}
where \(\hat\sigma\) is the odd probit sigmoid defined as in \eqref{eq:activation-relu}.
The idea of the proofs is to take dominated limits of the results in Appendix.~\ref{app:probit}, similar to Appendix~\ref{app:relu}'s treatment of Appendix~\ref{app:gelu}, so we omit them.
\begin{lemma}
  \label{lem:heaviside-MKL}
  Let \(\sigma\) be the Heaviside function as in \eqref{eq:activation-heaviside}.
  Then the functions defined in Def.~\ref{def:moment-maps} are
  \begin{align*}
    M_\sigma(\mu; \nu) &= \Phi\del{
      \frac{\mu}{\sqrt{\nu}}
    },
    \\
    K_\sigma(\mu_1, \mu_2; \nu_{11}, \nu_{22}, \nu_{12}) &=
    \left.\Phi_2\del{
      \frac{\mu_1}{\sqrt{\nu_{11}}},
      \frac{\mu_2}{\sqrt{\nu_{22}}};
      \rho'
    }
    \right|_{\rho' = 0}^{\rho' = \frac{\nu_{12}}{\sqrt{\nu_{11}\nu_{22}}}},
    \quad\text{and}
    \\
    L_\sigma(\mu_1; \nu_{11}, \nu_{22}, \nu_{12}) &= 
    \frac{\nu_{12}}{\sqrt{\nu_{11}}}
    \phi\del{
      \frac{\mu_1}{\sqrt{\nu_{11}}}
    }.
  \end{align*}
\end{lemma}

\clearpage
\section{Derivation of uncertainty propagation formulas for sine activation}
\label{app:sine}
In this appendix, we derive the \(M_\sigma\), \(K_\sigma\) and
\(L_\sigma\) functions (Def.~\ref{def:moment-maps}) for the sinusoidal activation function
\begin{align}
  \label{eq:activation-sin}
  \sigma(x) &= \sin(x).
\end{align}
We begin by recalling the identities that if \(Z \sim \mathcal N(\mu, \nu)\),
\begin{align}
\label{eq:expect-sin}
  \expect \sin(Z) &= e^{-\nu/2} \sin(\mu)
  \\
  \label{eq:expect-cos}
  \expect \cos(Z) &= e^{-\nu/2} \cos(\mu)
\end{align}
These formulas, which can be derived from the characteristic function of the normal distribution, allow moments of \(\sin(Z)\) and \(\cos(Z)\) to be computed exactly.
There is no need, as in \citet[Lemma~1.6]{sitzmann_implicit_2020} to use analytic approximations in this step.

From \eqref{eq:expect-sin} immediately follows:
\begin{lemma}
\label{lem:sine M}
  Let \(M_\sigma\) be defined as in Def.~\ref{def:moment-maps}, with
  \(\sigma\) as in \eqref{eq:activation-sin}.
  Then
  \begin{align*}
    M_\sigma(\mu; \nu) &= e^{-\nu/2}\sin(\mu).
  \end{align*}
\end{lemma}

\begin{lemma}
\label{lem:sine K}
  Let \(K_\sigma\) be defined as in Def.~\ref{def:moment-maps}, with
  \(\sigma\) as in \eqref{eq:activation-sin}.
  Then
  \begin{align*}
    \begin{split}
      K_\sigma(\mu_1, \mu_2; \nu_{11}, \nu_{22}, \nu_{12}) 
      &= \frac{1}{2} \sbr{e^{\nu^* + \nu_{12}} - e^{\nu^*}} \cos(\mu_1 - \mu_2) \\
      &\quad - \frac{1}{2} \sbr{e^{\nu^* - \nu_{12}} - e^{\nu^*}} \cos(\mu_1 + \mu_2),
    \end{split}
    \intertext{where}
    \nu^* &= -\frac{\nu_{11} + \nu_{22}}{2}.
  \end{align*}
\end{lemma}
\begin{proof}
  Let
  \begin{align}
  \begin{pmatrix} X_1 \\ X_2 \end{pmatrix}
   \sim \mathcal N\left(\begin{pmatrix}
    \mu_1
    \\
    \mu_2
  \end{pmatrix},
  \begin{pmatrix}
    \nu_{11}
    &
    \nu_{12}
    \\
    \nu_{12}
    &
    \nu_{22}
  \end{pmatrix}\right).
  \end{align}
  Then by inserting \eqref{eq:activation-sin}, we have
  \begin{align}
    \Cov(\sigma(X_1), \sigma(X_2))
    &= \expect \sin(X_1) \sin(X_2) - \expect \sin(X_1) \expect \sin(X_2).
    \label{eq:cov-sin-sin}
  \end{align}
  Using some trigonometric identities and \eqref{eq:expect-cos}, the first term of \eqref{eq:cov-sin-sin} becomes
  \begin{align}
    \expect \sin(X_1) \sin(X_2)
    &= \frac{1}{2} \expect \cos(\underbrace{X_1 - X_2}_{
      \mathcal N(\mu_1 - \mu_2, \nu_{11} + \nu_{22} - 2\nu_{12})
      })
    - \frac{1}{2} \expect \cos(\underbrace{X_1 + X_2}_{\mathcal N(\mu_1 + \mu_2, \nu_{11} + \nu_{22} + 2\nu_{12})})
    \\
    \begin{split}
      &= \frac{1}{2} \exp\del{-\frac{\nu_{11} + \nu_{22}}{2} + \nu_{12}} \cos(\mu_1 - \mu_2) \\
      &\quad - \frac{1}{2} \exp\del{ -\frac{\nu_{11} + \nu_{22}}{2} - \nu_{12}} \cos(\mu_1 + \mu_2)
    \end{split}
  \end{align}
  The second term of \eqref{eq:cov-sin-sin} becomes
  \begin{align}
    \begin{split}
      \expect \sin(X_1) \expect \sin(X_2)
      &= \frac{1}{2} \exp\del{-\frac{\nu_{11} + \nu_{22}}{2}} \cos(\mu_1 - \mu_2) \\
      &\quad - \frac{1}{2} \exp\del{ -\frac{\nu_{11} + \nu_{22}}{2}} \cos(\mu_1 + \mu_2)
    \end{split}
  \end{align}
  The result follows from collecting like terms.
\end{proof}

\begin{lemma}
\label{lem:sine L}
  Let \(L_\sigma\) be defined as in Def.~\ref{def:moment-maps}, with
  \(\sigma\) as in \eqref{eq:activation-sin}.
  Then
  \begin{align*}
    L_\sigma(\mu_1, \mu_2; \nu_{11}, \nu_{22}, \nu_{12})
    &=
    \nu_{12} e^{-\nu_{11}/2} \cos(\mu_1). 
  \end{align*}
\end{lemma}
\begin{proof}
   Let
  \begin{align}
  \begin{pmatrix} X_1 \\ X_2 \end{pmatrix}
   \sim \mathcal N\left(\begin{pmatrix}
    \mu_1
    \\
    \mu_2
  \end{pmatrix},
  \begin{pmatrix}
    \nu_{11}
    &
    \nu_{12}
    \\
    \nu_{12}
    &
    \nu_{22}
  \end{pmatrix}\right).
  \end{align}
  Using Lemma \ref{lem:stein}, we have
  \begin{align}
    L_\sigma(\mu_1, \mu_2; \nu_{11}, \nu_{22}, \nu_{12})
    &= \nu_{12} \expect \sigma'(X_1)
    \\
    &= \nu_{12} \expect \cos(X_1)
    \\
    &= \nu_{12} e^{-\nu_{11}/2} \cos(\mu_1)
  \end{align}
  by \eqref{eq:expect-cos}.
\end{proof}

\clearpage
\section{Theoretical Guarantees}
\label{app:theoretical-guarantees}
This section proves the smooth-distance bounds stated in \S\ref{sec:theoretical-guarantees}.
We first prove the single-layer higher-order expansion, and then prove that the smooth distance satisfies the layerwise recursion.

\subsection{Single-layer smooth-distance expansion}

Let \(X=\mu_X+\lambda^{-1}\xi\), where \(\xi\sim\mathcal N(0,\Sigma_X)\), and set \(\epsilon=\lambda^{-1}\).
Let \(Y=f(X)\), \(\mu_Y=\expect Y\), \(\Sigma_Y=\Cov Y\), and \(\gamma_Y\sim\mathcal N(\mu_Y,\Sigma_Y)\).
For \(h\) satisfying \(\left\|h\right\|_{\dot W^{4,\infty}}\leq 1\), Taylor expansion around \(\mu_Y\) gives
\begin{align*}
  h(y)
  &=
  \underbrace{h(\mu_Y)}_{(\mathrm{I})}
  +
  \underbrace{Dh(\mu_Y)[y-\mu_Y]}_{(\mathrm{II})}
  +
  \underbrace{\frac12 D^2h(\mu_Y)[y-\mu_Y,y-\mu_Y]}_{(\mathrm{III})}
  \\
  &\quad+
  \underbrace{\frac{1}{6}D^3h(\mu_Y)[y-\mu_Y,y-\mu_Y,y-\mu_Y]}_{(\mathrm{IV})}
  +
  \underbrace{R_4(y)}_{(\mathrm{V})},
\end{align*}
with \(|R_4(y)|\lesssim \|y-\mu_Y\|^4\).
Term \((\mathrm{I})\) is deterministic and therefore cancels when subtracting \(\expect h(\gamma_Y)\) from \(\expect h(Y)\).
Term \((\mathrm{II})\) cancels because \(Y\) and \(\gamma_Y\) have the same mean \(\mu_Y\).
Term \((\mathrm{III})\) cancels because \(Y\) and \(\gamma_Y\) have the same covariance \(\Sigma_Y\).
Term \((\mathrm{V})\) is controlled by the fourth central moments:
\begin{align*}
  \expect \|Y-\mu_Y\|^4 + \expect \|\gamma_Y-\mu_Y\|^4
  =
  O(\epsilon^4),
\end{align*}
because \(Y-\mu_Y=O_p(\epsilon)\) and \(\Sigma_Y=O(\epsilon^2)\).
It remains to identify the size of term \((\mathrm{IV})\), the cubic term.

Taylor expansion of \(f\) around \(\mu_X\) gives
\begin{align*}
  Y
  =
  \underbrace{f(\mu_X)}_{(\mathrm{A})}
  +
  \underbrace{\epsilon J\xi}_{(\mathrm{B})}
  +
  \underbrace{\epsilon^2 Q(\xi)}_{(\mathrm{C})}
  +
  \underbrace{O_p(\epsilon^3)}_{(\mathrm{D})},
\end{align*}
where \(J=\nabla f(\mu_X)^\intercal\) and \(Q(\xi)=\frac12\nabla^2 f(\mu_X)[\xi,\xi]\).
Taking expectations gives
\begin{align*}
  \mu_Y
  =
  \underbrace{f(\mu_X)}_{(\mathrm{A})}
  +
  \epsilon^2\expect Q(\xi)
  +
  O(\epsilon^4),
\end{align*}
because the centered Gaussian term \((\mathrm{B})\) has mean zero and the cubic-order term in \((\mathrm{D})\) is odd to leading order.
Consequently, subtracting \(\mu_Y\) cancels term \((\mathrm{A})\) and centers term \((\mathrm{C})\):
\begin{align*}
  Y-\mu_Y
  =
  \underbrace{\epsilon J\xi}_{(\mathrm{B})}
  +
  \underbrace{\epsilon^2\del{Q(\xi)-\expect Q(\xi)}}_{(\mathrm{C})-\expect(\mathrm{C})}
  +
  \underbrace{O_p(\epsilon^3)}_{(\mathrm{D})}.
\end{align*}
The cubic contribution in term \((\mathrm{IV})\) is linear in \(\expect(Y-\mu_Y)^{\otimes 3}\).
Using the previous display,
\begin{align*}
  (Y-\mu_Y)^{\otimes 3}
  &=
  \underbrace{(\mathrm{B})^{\otimes 3}}_{\epsilon^3}
  +
  \underbrace{
    \operatorname{Sym}\del{(\mathrm{B})^{\otimes 2}\otimes\del{(\mathrm{C})-\expect(\mathrm{C})}}
  }_{\epsilon^4}
  +
  O_p(\epsilon^5),
\end{align*}
where \(\operatorname{Sym}\) denotes the sum over the three placements of the \((\mathrm{C})-\expect(\mathrm{C})\) factor.
The first term has expectation zero because \((\mathrm{B})=\epsilon J\xi\) is centered Gaussian.
The displayed symmetrized term is therefore the first possible nonzero contribution: it contains two factors of \((\mathrm{B})\), each of size \(O(\epsilon)\), and one factor of \((\mathrm{C})-\expect(\mathrm{C})\), of size \(O(\epsilon^2)\).
All remaining terms contain either one factor of \((\mathrm{D})=O_p(\epsilon^3)\) or at least two factors of \((\mathrm{C})-\expect(\mathrm{C})\), and hence are \(O_p(\epsilon^5)\) or smaller.
Thus
\begin{align*}
  \expect (Y-\mu_Y)^{\otimes 3}
  =
  O(\epsilon^4).
\end{align*}
Since \(\gamma_Y\) is Gaussian, its third central moment is zero.
Thus term \((\mathrm{IV})\) contributes only \(O(\epsilon^4)\) to \(\expect h(Y)-\expect h(\gamma_Y)\).
Therefore, uniformly over \(\left\|h\right\|_{\dot W^{4,\infty}}\leq 1\),
\begin{align*}
  \left|
  \expect h(Y)-\expect h(\gamma_Y)
  \right|
  =
  O(\epsilon^4)
  =
  O(\lambda^{-4}).
\end{align*}
Taking the supremum over the homogeneous unit ball of \(\dot W^{4,\infty}\) proves Theorem~\ref{thm:smooth-single-layer}.

The same expansion also explains the comparison with linearization.
The linearized Gaussian \(G_\mathrm{lin}\) has mean \(f(\mu_X)\), whereas
\begin{align*}
  \mu_Y
  =
  f(\mu_X)
  +
  \epsilon^2\expect Q(\xi)
  +
  O(\epsilon^4).
\end{align*}
For test functions with gradient aligned with \(\expect Q(\xi)\), this missing second-order mean correction produces a generic \(O(\epsilon^2)=O(\lambda^{-2})\) smooth-distance error.
Standard fixed sigma-point approximations that do not exactly match the true layer mean and covariance have the same obstruction in this expansion.

\subsection{Layerwise recursion in smooth distance}

We recall the layer-by-layer Normal approximation resulting in \(Y_\mathrm{ana}\) (henceforth, \(Y\)) alongside the exact neural network formula resulting in \(Y_\mathrm{true}\):
\begin{subequations}
  \begin{align}
    Y &= Y^\ell & Y_\mathrm{true} &= Y_\mathrm{true}^\ell
    \\
    Y^k &= \normal g^k(Y^{k-1}), & Y_\mathrm{true}^k &= g^k(Y_\mathrm{true}^{k-1}), & k \in
    \cbr{1 \ldots \ell},
    \label{eq:hidden-layer-comparison}
    \\
    Y^0 &= X & Y_\mathrm{true}^0 &= X,
    \intertext{where \(g^k\) is the function}
    g^k(x) &= g(x; A^k, b^k, C^k, d^k).
\end{align}
\end{subequations}
The smooth distance is an integral probability metric, so for any random variables \(U,V,W\),
\begin{align*}
  d_4(U,W)
  \leq
  d_4(U,V)+d_4(V,W).
\end{align*}
Moreover, for any \(C^4\) map \(g\),
\begin{align*}
  d_4(g(U),g(V))
  &=
  \sup_{\left\|h\right\|_{\dot W^{4,\infty}}\leq 1}
  \left|
  \expect h(g(U))-\expect h(g(V))
  \right|
  \\
  &\leq
  C_4(g)d_4(U,V),
\end{align*}
where
\begin{align*}
  C_4(g)
  =
  \sup_{\left\|h\right\|_{\dot W^{4,\infty}}\leq 1}
  \left\|h\circ g\right\|_{\dot W^{4,\infty}}.
\end{align*}
For a layer \(g(x;A,b,C,d)=\sigma(Ax+b)+Cx+d\), the Fa\`a di Bruno formula gives
\begin{align*}
  C_4(g)
  \leq
  P_4\del{\|A\|,\|C\|,\|\sigma'\|_\infty,\ldots,\|\sigma^{(4)}\|_\infty},
\end{align*}
for a polynomial \(P_4\) depending only on the input and output dimensions.

Define
\begin{align}
  \Delta_4^k &:= d_4(Y_\mathrm{true}^k, Y^k),
  &
  E_{4,k} &:= d_4(g^k(Y^{k-1}),Y^k).
  \label{eq:smooth-discrepancy}
\end{align}
Then the basic triangle inequality is
\begin{align}
  \Delta_4^k
  &= d_4(g^k(Y_\mathrm{true}^{k-1}), Y^k)
  \tag{definition of hidden layer \eqref{eq:hidden-layer-comparison}}
  \\
  &\leq d_4(g^k(Y_\mathrm{true}^{k-1}), g^k(Y^{k-1})) + d_4(g^k(Y^{k-1}), Y^k)
  \tag{triangle inequality}
  \\
  &\leq C_4(g^k)d_4(Y_\mathrm{true}^{k-1}, Y^{k-1}) + d_4(g^k(Y^{k-1}), Y^k)
  \tag{composition property of \(d_4\)}
  \\
  &\leq C_4(g^k)\Delta_4^{k-1} + E_{4,k}.
  \label{eq:smooth-discrepancy-recursion}
\end{align}
If \(X\) is Gaussian, then \(\Delta_4^0 = 0\).
The forcing term \(E_{4,k}=d_4(g^k(Y^{k-1}),Y^k)\) is exactly the single-layer Gaussianization error for the Normal input \(Y^{k-1}\).
In the low-variance regime, Theorem~\ref{thm:smooth-single-layer} gives \(E_{4,k}=O(\lambda^{-4})\).
Therefore
\begin{align}
  \Delta_4^k \leq C_4(g^k) \Delta_4^{k-1} + E_{4,k}.
\end{align}
Unrolling this recursion and using \(\Delta_4^0 = 0\) yields
\begin{align*}
  \Delta_4^\ell
  \leq
  \sum_{k=1}^{\ell}
  \del{\prod_{r=k+1}^{\ell} C_4(g^r)}
  E_{4,k},
\end{align*}
which proves Theorem~\ref{thm:recursive-smooth-bound}.

\clearpage
\section{Supplement to Fig.~\ref{fig:deep-sine-convergence}}
\label{app:deep-sine-convergence}
The convergence experiment in Fig.~\ref{fig:deep-sine-convergence} is generated by \texttt{demo/deep\_sine\_convergence.py}.
The purpose is to visualize the small-input-covariance regime in a fixed nonlinear network where both exact Gaussian moment matching and linearization can be compared against a quasi-Monte Carlo pseudo-truth.

The network is a randomly initialized deep sine network generated by the same \texttt{build\_network} utility used in the random-network experiments.
Specifically, the script fixes the JAX random seed to \(1\), uses input dimension \(3\), chooses the \texttt{deep} architecture, and uses the sine activation.
For a covariance scale \(c>0\), the input distribution is
\begin{align*}
  X_c \sim \mathcal N(0,cI_3).
\end{align*}
In the full experiment, the covariance scales are logarithmically spaced as
\begin{align*}
  c \in \operatorname{geomspace}(10^{-1},10^{-8},8).
\end{align*}
The script also includes a faster diagnostic mode using \(5\) covariance scales from \(10^{-1}\) to \(10^{-5}\).

For each covariance scale, one quasi-Monte Carlo run is performed.
The full experiment uses \(2^{20}\) samples; the fast diagnostic mode uses \(2^{12}\) samples.
These samples are propagated through the deterministic neural network to produce samples from \(Y_\mathrm{true}=f(X_c)\).
The pseudo-true Gaussian is then
\begin{align*}
  Y_\mathrm{pseudo}
  =
  \mathcal N\del{\widehat{\expect}Y_\mathrm{true},\widehat{\Var}Y_\mathrm{true}},
\end{align*}
where the mean and variance are estimated from the quasi-Monte Carlo output samples.

The two approximations compared are
\begin{itemize}
  \item \(Y_\mathrm{ana}\), produced by analytic layerwise Gaussian moment matching, and
  \item \(Y_\mathrm{lin}\), produced by linearized covariance propagation.
\end{itemize}
For each method, the reported error is
\begin{align*}
  D_{\mathrm{KL}}(Y_\mathrm{pseudo}\parallel Y_\mathrm{method}).
\end{align*}
The raw results, sorted summary table, and figure are written to \texttt{docs/manuscript/generated/convergence/}.

\clearpage
\section{Supplement to \S\ref{sec:bayesian-networks}}
\label{app:bayesian-networks}
We expand on each of the terms in the evidence lower bound \eqref{eq:elbo}, repeated below.
\begin{align*}
  \theta^* =
  \arg\min_\theta \cbr{
    -\expect _{w \sim q(w; \theta)} \log p(y \mid w, x) + D_{\text{KL}, w}(q(w;\theta) \mid p(w))
  }.
\end{align*}
Our neural network \(f(x; w)\) consists of a single hidden layer.
The weights \(w\) encompass:
\begin{itemize}
  \item \(A\), \(b\): pre-activation weights,
  \item \(C\), \(d\): post-activation weights, as well as
  \item \(\lambda\): (homoscedastic) output log precision.
\end{itemize}
The log likelihood of \(w\) is Normal:
\begin{align*}
  \log p(y \mid w, x)
  &= \frac{1}{2} \del{
    -\lambda
    + e^\lambda \sbr{y - f(x;w )}^{2}
  }
  + \frac{1}{2} \log \sbr{ 2\pi } .
\end{align*}
The prior distribution \(p(w)\) is the Kaiming initialization of independent Normal distributions, zero mean in each matrix entry and variance equal to \(2 / \text{fan-in}\).
Like \citet{wu_deterministic_2019,petersen_uncertainty_2024,wright_analytic_2024}, the variational distribution \(q(w; \theta)\) is a Gaussian matrix with independent entries, parameterized by means and log precisions.

We train using the Adam optimizer with a learning rate of 0.1 for 10\,000 iterations.
Each gradient estimate averages 10 stochastic ELBO estimates, each drawing one random sample from the variational distribution \(q(w; \theta)\) per training instance.

The implementation uses Equinox and JAX \citep{kidger_equinox_2021,deepmind_deepmind_2020}. 
Each network takes about 2 minutes on an Nvidia T1200 GPU.

In reporting, we draw one million random variates from the variational posterior distribution for up to 1000 instances from the test set.

The datasets are
\begin{itemize}
  \item \href{https://archive.ics.uci.edu/dataset/87/servo}{Servo} \citep{ulrich_servo_1986} (CC BY 4.0)
  \item \href{https://archive.ics.uci.edu/dataset/60/liver+disorders}{Liver disorders} \citep{unknown_liver_2016}
  \item \href{https://archive.ics.uci.edu/dataset/9/auto+mpg}{Auto MPG} \citep{r_quinlan_auto_1993} (CC BY 4.0)
  \item \href{https://archive.ics.uci.edu/dataset/477/real+estate+valuation+data+set}{Real estate valuation} \citep{i-cheng_yeh_real_2018} (CC BY 4.0)
  \item \href{https://archive.ics.uci.edu/dataset/162/forest+fires}{Forest fires}
  \citep{paulo_cortez_forest_2007} (CC BY 4.0)
  \item \href{https://archive.ics.uci.edu/dataset/925/infrared+thermography+temperature}{Infrared thermography temperature}
  \cite{wang_facial_nodate} (CC0)
  \item \href{https://archive.ics.uci.edu/dataset/294/combined+cycle+power+plant}{Combined cycle power plant} \citep{pnar_tfekci_combined_2014}
  (CC BY 4.0)
  \item \href{https://archive.ics.uci.edu/dataset/165/concrete+compressive+strength}{Concrete compressive strength} \citep{i-cheng_yeh_concrete_1998}
  (CC BY 4.0)
  \item \href{https://archive.ics.uci.edu/dataset/291/airfoil+self+noise}{Airfoil self-noise}
  \citep{thomas_brooks_airfoil_1989}
  (CC BY 4.0)
  \item \href{https://archive.ics.uci.edu/dataset/183/communities+and+crime}{Communities and crime}
  \citep{redmond_communities_2002}
  (CC BY 4.0)
  \item \href{https://archive.ics.uci.edu/dataset/189/parkinsons+telemonitoring}{Parkinsons telemonitoring}
  \citep{athanasios_tsanas_parkinsons_2009}
  (CC BY 4.0)
\end{itemize}
\paragraph{Impact/ethics statement}
Some of these datasets, served on the UC Irvine Machine Learning Repository, 
are based on observing humans.
This choice of datasets follows \citet{rui_li_streamlining_2025}.
While we defer to UCI's custodianship and ICLR 2025's review process in deeming them ethically appropriate for methodology research, 
we are open to evidence to the contrary.

\clearpage
\subsection{Results}

\begin{figure}
  \centering
  \includegraphics[width=\linewidth]{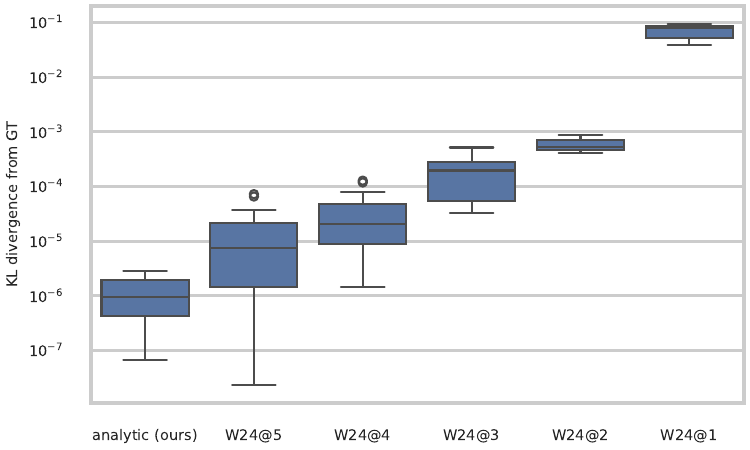}
  \caption{KL divergence between ground truth predictive distribution (by Monte Carlo) and approximations for the Servo dataset. W24@$k$ means \citet{wright_analytic_2024} with $k$ terms in the series expansion.}
\end{figure}

\begin{figure}
  \centering
  \includegraphics[width=\linewidth]{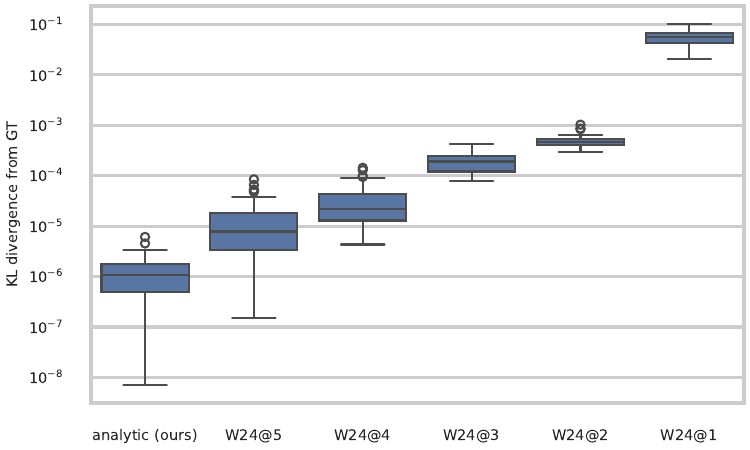}
  \caption{KL divergence between ground truth predictive distribution (by Monte Carlo) and approximations for the liver disorders dataset. W24@$k$ means \citet{wright_analytic_2024} with $k$ terms in the series expansion.}
\end{figure}

\begin{figure}
  \centering
  \includegraphics[width=\linewidth]{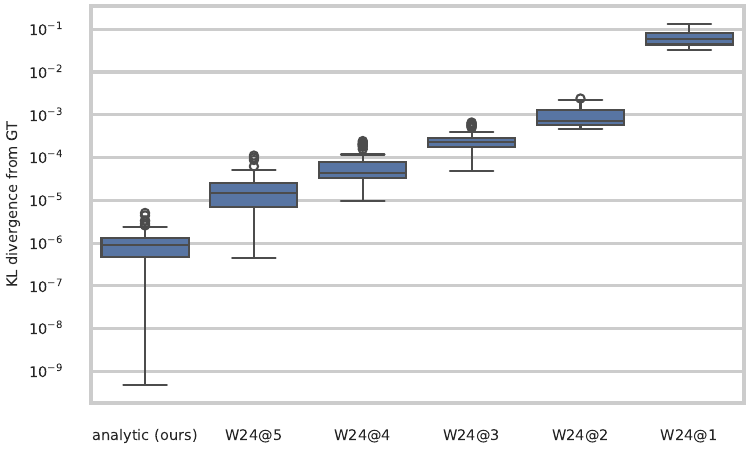}
  \caption{KL divergence between ground truth predictive distribution (by Monte Carlo) and approximations for the Auto MPG dataset. W24@$k$ means \citet{wright_analytic_2024} with $k$ terms in the series expansion.}
\end{figure}

\begin{figure}
  \centering
  \includegraphics[width=\linewidth]{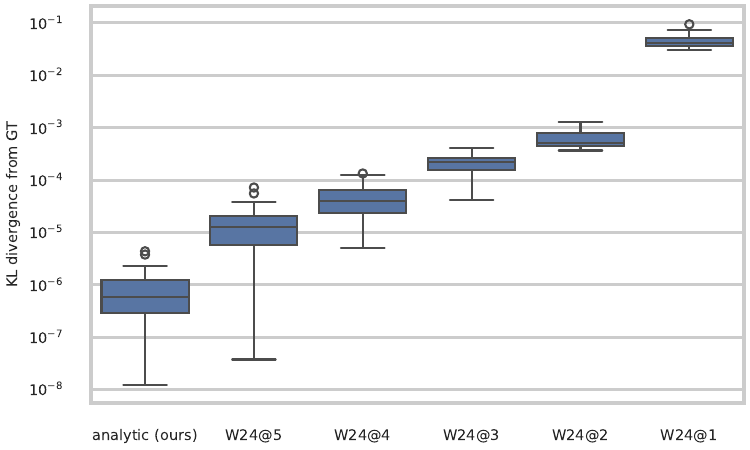}
  \caption{KL divergence between ground truth predictive distribution (by Monte Carlo) and approximations for the real estate valuation dataset. W24@$k$ means \citet{wright_analytic_2024} with $k$ terms in the series expansion.}
\end{figure}

\begin{figure}
  \centering
  \includegraphics[width=\linewidth]{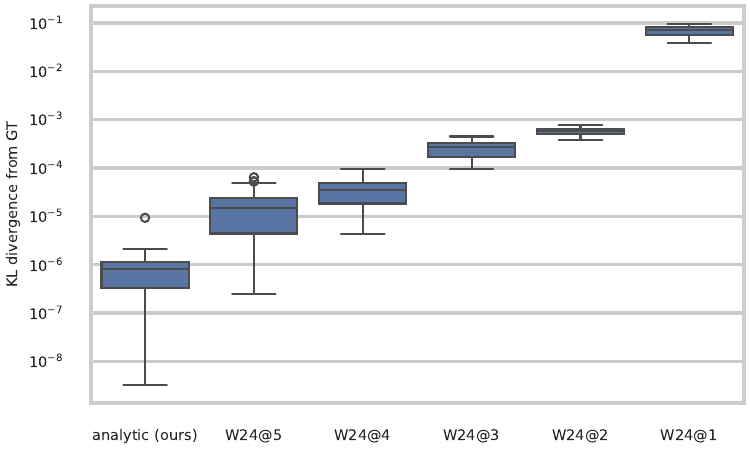}
  \caption{KL divergence between ground truth predictive distribution (by Monte Carlo) and approximations for the forest fires dataset. W24@$k$ means \citet{wright_analytic_2024} with $k$ terms in the series expansion.}
\end{figure}

\begin{figure}
  \centering
  \includegraphics[width=\linewidth]{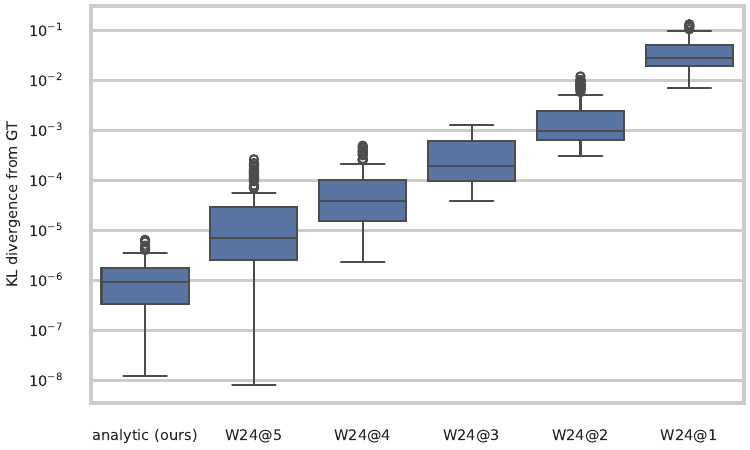}
  \caption{KL divergence between ground truth predictive distribution (by Monte Carlo) and approximations for the infrared thermography temperature dataset. W24@$k$ means \citet{wright_analytic_2024} with $k$ terms in the series expansion.}
\end{figure}

\begin{figure}
  \centering
  \includegraphics[width=\linewidth]{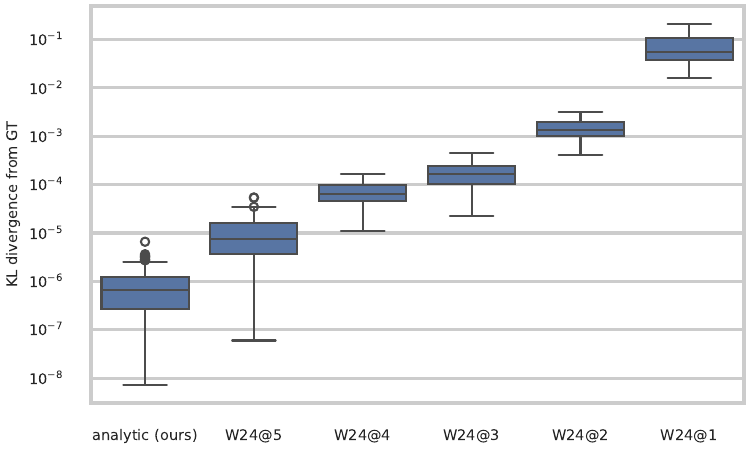}
  \caption{KL divergence between ground truth predictive distribution (by Monte Carlo) and approximations for the concrete compressive strength dataset. W24@$k$ means \citet{wright_analytic_2024} with $k$ terms in the series expansion.}
\end{figure}

\begin{figure}
  \centering
  \includegraphics[width=\linewidth]{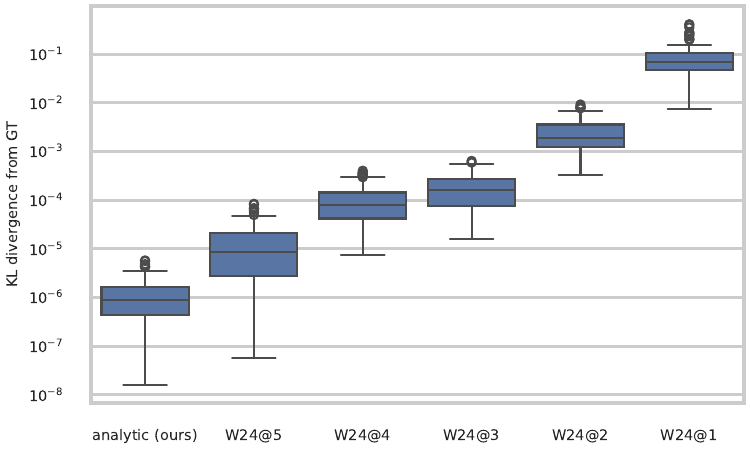}
  \caption{KL divergence between ground truth predictive distribution (by Monte Carlo) and approximations for the airfoil self-noise dataset. W24@$k$ means \citet{wright_analytic_2024} with $k$ terms in the series expansion.}
\end{figure}

\begin{figure}
  \centering
  \includegraphics[width=\linewidth]{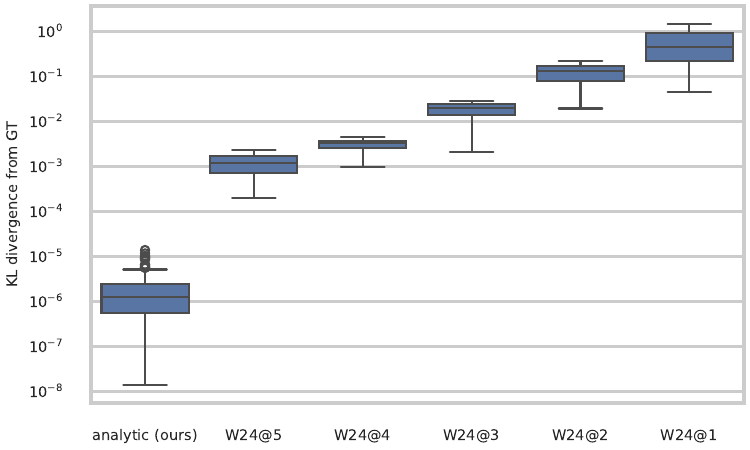}
  \caption{KL divergence between ground truth predictive distribution (by Monte Carlo) and approximations for the communities and crime dataset. W24@$k$ means \citet{wright_analytic_2024} with $k$ terms in the series expansion.}
\end{figure}

\begin{figure}
  \centering
  \includegraphics[width=\linewidth]{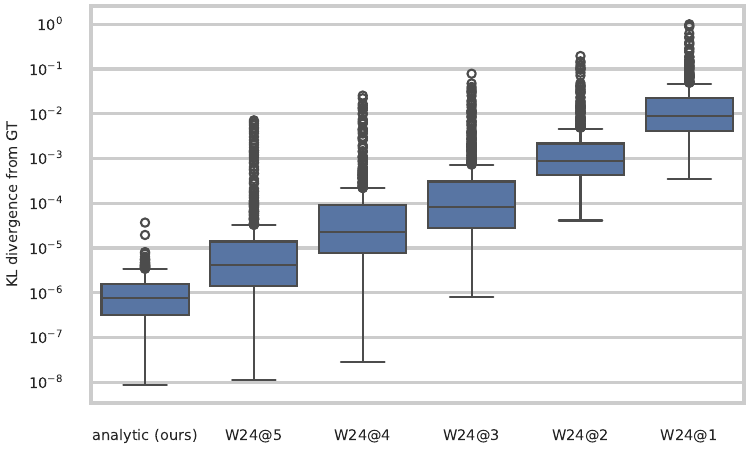}
  \caption{KL divergence between ground truth predictive distribution (by Monte Carlo) and approximations for the Parkinsons telemonitoring dataset. W24@$k$ means \citet{wright_analytic_2024} with $k$ terms in the series expansion.}
\end{figure}

\begin{figure}
  \centering
  \includegraphics[width=\linewidth]{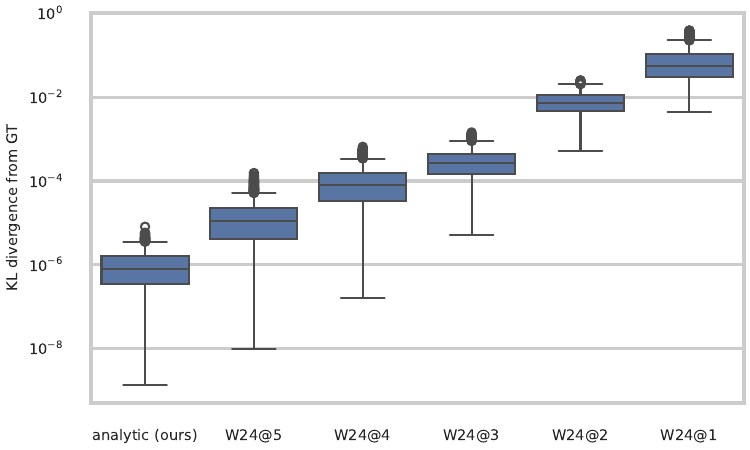}
  \caption{KL divergence between ground truth predictive distribution (by Monte Carlo) and approximations for the combined cycle power plant dataset. W24@$k$ means \citet{wright_analytic_2024} with $k$ terms in the series expansion.}
\end{figure}

\clearpage

\section{Supplement to \S\ref{sec:stochastic-networks}}
\label{app:stochastic-networks}
The random variable \(\bm{1}_{U < \Phi(x)}\) has a Bernoulli distribution with parameter \(\Phi(x)\),
which contributes an additional diagonal variance term to Lemma~\ref{lem:moment-maps} as we see from applying the Law of Total Covariance:
\begin{multline}
  \del{\Cov g_{\tilde \sigma}(X; A, b, C, d)}_{i, j}
  =
    \del{\Cov \expect \sbr{g_{\tilde \sigma}(X; A, b, C, d) \mid X}}_{i, j}
    \\
    +
    \del{\expect \Cov \sbr{g_{\tilde \sigma}(X; A, b, C, d) \mid X}}_{i, j}
\end{multline}
\begin{align}
\del{\Cov \expect \sbr{g_{\tilde \sigma}(X; A, b, C, d) \mid X}}_{i, j}
&= \del{\Cov g_{\sigma}(X; A, b, C, d)}_{i, j}
  \intertext{since \(\expect \del{\tilde \sigma(X, U) \mid X} = \sigma(X)\) and}
  \del{\expect \Cov \sbr{g_{\tilde \sigma}(X; A, b, C, d) \mid X}}_{i, j}
  &=
    \delta_{ij}
    \expect \Var \sbr{\del{g_{\tilde \sigma}(X; A, b, C, d)}_i \mid X}
    \\
    &= 4\delta_{ij} \expect \Phi(\xi_i) \del{1 - \Phi(\xi_i)}
    \\
    \intertext{where \(\xi_i \sim \mathcal N(\mu_i, \nu_{ii})\).
    This can be expressed using Owen's T function,}
    \expect \Phi(\xi_i) \del{1 - \Phi(\xi_i)}
    &= 2 T\del{
      \frac{\mu_i}{\sqrt{1 + \nu_{ii}}},
      \frac{1}{\sqrt{1 + 2\nu_{ii}}}
    },
\end{align}
which satisfies \citep[equation 20,010.4]{owen_table_1980}
\begin{align*}
  \int \Phi(\mu + \sqrt \nu z)^2 \phi(z) dz
  =
  \Phi\left(\frac{\mu}{\sqrt{1+\nu}}\right)
  -
  2T\left(
  \frac{\mu}{\sqrt{1+\nu}},
  \frac{1}{\sqrt{1+2\nu}}
  \right).
\end{align*}

\begin{figure}
  \centering
  \includegraphics[width=.7\linewidth]{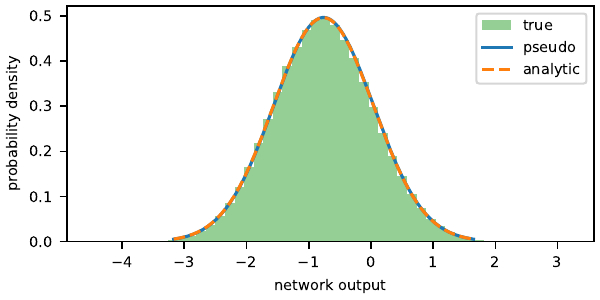}
  \caption{\label{fig:stochastic} Output distribution of a stochastic neural network, pseudo-true Normal distribution (``pseudo''), and layer-by-layer moment-matched Normal distribution (``analytic'').}
\end{figure}

\clearpage
\section{Supplement to \S\ref{sec:random-networks}}
\label{app:random-neural-networks}
We apply our method and other benchmarks to the random-network ensembles generated by \texttt{demo/test\_case.py}.
The script iterates over \(36\) ensembles:
\begin{itemize}
  \item network architecture \( \in \{\text{wide}, \text{deep}\} \)
  \item weights \( \in \{\text{initialized}, \text{trained}\} \)
  \item activation function \(\in\) \{probit, probit residual, sine, sine residual, GeLU, GeLU residual, ReLU, ReLU residual, Heaviside, Heaviside residual\}.
\end{itemize}
(The trained Heaviside and Heaviside-residual ensembles are skipped.)
From each ensemble, we sample one neural network and evaluate the goodness of approximation of the output distribution for input distributions: 
\begin{itemize}
  \item variance \( \in \{\text{small}, \text{medium}, \text{large}\} \).
\end{itemize}
In each case we compare the distributions of:
\begin{itemize}
  \item \(Y_\mathrm{true}\), the true distribution
  \item \(Y_\mathrm{pseudo}\), the pseudo-true Gaussian distribution
  \item \(Y_\mathrm{ana}\), the analytic layer-wise Gaussian approximation (our method)
  \item \(Y_\mathrm{mfa}\), mean-field: applying our method for moment propagation and setting off-diagonal layer covariances to zero
  \item \(Y_\mathrm{lin}\), linearization-based moment propagation
  \item \(Y_\mathrm{u'95}\), unscented transform of the whole network using \(\alpha=1, \beta=0, \kappa=0\)
  \item \(Y_\mathrm{u'02}\), unscented transform of the whole network using \(\alpha=0.001, \beta=2, \kappa=0\)
\end{itemize}
\subsection*{Ensembles of neural networks}
This section specifies the \(36\) ensembles of random neural networks and the three ensembles of inputs used to produce the \(108\) test cases in the accompanying \emph{Supplementary Numerical Results}.
The neural networks in this example are parameterized as in Def.~\ref{def:neural-network}.
Let \(d_\mathrm{hidden}\) and \(w_\mathrm{hidden}\) be the depth and width of the hidden layers, respectively.
A random neural network \(f:\mathbb{R}^3 \to \mathbb{R}\) is specified with \(d_\mathrm{hidden} + 1\) layers.
The output layer is linear, with weights \(C^{d_\mathrm{hidden}+1}\) and biases \(d^{d_\mathrm{hidden}+1}\).
The first hidden layer has \(C^1 = 0_{w_\mathrm{hidden} \times 3}\), \(d^1 = 0_{w_\mathrm{hidden}}\).
The four degrees of freedom in the test cases are:
\begin{description}
  \item[architecture] 
  \begin{itemize}
    \item if architecture = wide, then \(d_\mathrm{hidden} = 5\), \(w_\mathrm{hidden} = 400\);
    \item if architecture = deep, then \(d_\mathrm{hidden} = 20\), \(w_\mathrm{hidden} = 100\).
  \end{itemize}
  \item[weights]
  \begin{itemize}
    \item if weights = initialized:
    \begin{itemize}
      \item \(A\) matrices are initialized with i.i.d.~Gaussian entries having mean zero and variance equal to \(2\) divided by the number of columns
      \item if \(\text{activation} \in \{\text{probit}, \text{probit residual}\}\), then \(b\) vectors are initialized with independently sampled entries from \(\mathcal N(0, 1)\)
      \item if \(\text{activation} \in \{\text{sine}, \text{sine residual}\}\), then \(b\) vectors are initialized with independently sampled entries from \(\mathcal U(-\pi, \pi)\)
      \item if \(\text{activation} \in \{\text{GeLU}, \text{GeLU residual}, \text{ReLU}, \text{ReLU residual}, \text{Heaviside}, \text{Heaviside residual}\}\), then \(b\) vectors are initialized with independently sampled entries from \(\mathcal N(0, 1)\)
      \item if the activation is non-residual, then hidden-layer \(C\) matrices are initialized to the zero matrix.
      \item if the activation is residual, then square hidden-layer \(C\) matrices after the first hidden layer are initialized to the identity matrix, and all other hidden-layer \(C\) matrices are initialized to the zero matrix.
      \item the output-layer \(C\) matrix is initialized with i.i.d.~Gaussian entries having mean zero and variance equal to the reciprocal of the number of columns.
      \item \(d\) vectors are initialized to the zero vector.
    \end{itemize}
    \item if weights = trained and \(\text{activation} \not\in \{\text{Heaviside}, \text{Heaviside residual}\}\), after initialization as above, the neural network is trained to minimize the mean squared error loss on a pseudorandomly generated dataset of ten \((x, y)\) samples drawn from \(\mathcal N(0, 1)\).
    Training consists of using the AdamW optimizer \citep{loshchilov_decoupled_2019} with a learning rate of \(10^{-6}\) for 30,000 iterations and until the loss is less than \(10^{-8}\) (whichever is later).
    Implementation is due to Optax, \citet{deepmind_deepmind_2020}.
  \end{itemize}
  \item[activation function] 
  \begin{itemize}
    \item if \(\text{activation} \in \{\text{probit}, \text{probit residual}\}\), then \(\sigma(x) = 2\Phi(x) - 1\) where \(\Phi\) is the cumulative distribution function of the standard normal distribution
    \item if \(\text{activation} \in \{\text{sine}, \text{sine residual}\}\), then \(\sigma(x) = \sin(x)\)
    \item if \(\text{activation} \in \{\text{GeLU}, \text{GeLU residual}\}\), then \(\sigma(x) = x \Phi(x)\).
    \item if \(\text{activation} \in \{\text{ReLU}, \text{ReLU residual}\}\), then \(\sigma(x) = \max(0, x)\).
    \item if \(\text{activation} \in \{\text{Heaviside}, \text{Heaviside residual}\}\), then \(\sigma(x) = \bm{1}_{\{x \geq 0\}}\)
  \end{itemize}
  \item[variance]
  \begin{itemize}
    \item if variance = small, then \(X \sim \mathcal N(0, 10^{-2}I)\)
    \item if variance = medium, then \(X \sim \mathcal N(0, I)\)
    \item if variance = large, then \(X \sim \mathcal N(0, 10^2I)\)
  \end{itemize}
\end{description}

\subsection*{Simulation and reporting}
The only stochastic uncertainty in these numerical results comes from estimating \(Y_\mathrm{true}=f(X)\) and \(Y_\mathrm{pseudo}\).
For each test case, \texttt{demo/test\_case.py} uses twenty independently scrambled quasi-Monte Carlo realizations of \(N=2^{16}\) samples \citep{virtanen_scipy_2020}.
The pseudo-true Gaussian \(Y_\mathrm{pseudo}\) is recomputed from each realization using the sample mean and covariance.
Tables report statistical uncertainty in the form \(\text{mean} \pm \text{standard error}\) across the twenty realizations.
Distribution plots pool the twenty realizations and display the resulting \(Y_\mathrm{true}\) sample with a histogram.

The reported Wasserstein statistic between a Normal distribution and \(Y_\mathrm{true}\) is computed by
\begin{align*}
  d_\mathrm{W}(\mathcal N(\mu, \sigma^2), Y_\mathrm{true}) 
  &\approx s_Y^{-1/2}\frac{1}{N}\sum_{i=1}^N \left|y_{(i)} - Q\del{\frac{i -1/2}{N}}\right|
\end{align*}
where \(\{y_{(i)}\}_{i=1}^N\) are the quasi-Monte Carlo samples of \(Y_\mathrm{true}\) sorted in ascending order, \(Q\) is the quantile function of \(\mathcal N(\mu, \sigma^2)\), and \(s_Y\) is the sample standard deviation of the same realization.
When plotting distributions, we show \(Y_\mathrm{true}\) using a histogram with 50 bins and overlay the pseudo-true, analytic, mean-field, linear, and two unscented Gaussian densities.

The horizontal axes are scaled to include the 0.5th (99.5th) percentiles of the quasi-Monte Carlo samples of \(Y_\mathrm{true}\), or the mean minus (plus) three standard deviations of \(Y_\mathrm{pseudo}\), whichever is smaller (greater).

The entire suite of 108 test cases is laptop-scale; including initialization, training, quasi-Monte Carlo, and reporting, a representative run takes 1--2 hours on an Ubuntu system with an 11th Gen Intel® Core™ i7-11850H CPU and 32GB of RAM.

\clearpage
\subsection{Summaries}
\label{sec:random-neural-networks-summaries}
\begin{figure}[H]
\begin{center}
  \includegraphics{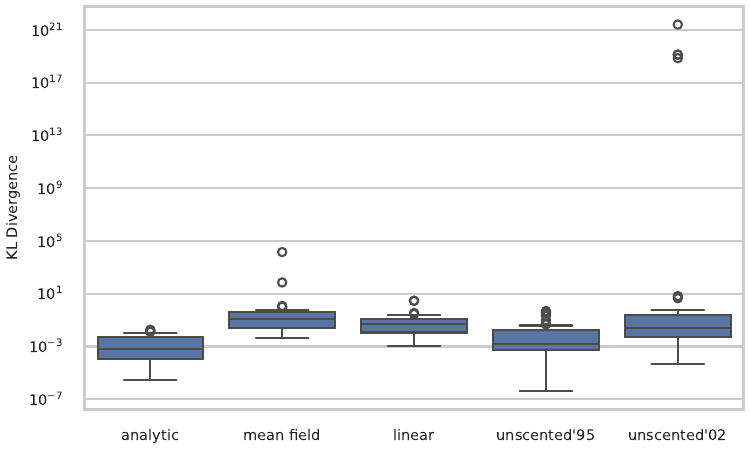}
\end{center}
\caption{Comparison of goodness of approximation (lower KL divergence is better) for all random neural networks, grouped by approximation method, in the small input variance scenario.}
\end{figure}
\begin{figure}[H]
\begin{center}
  \includegraphics{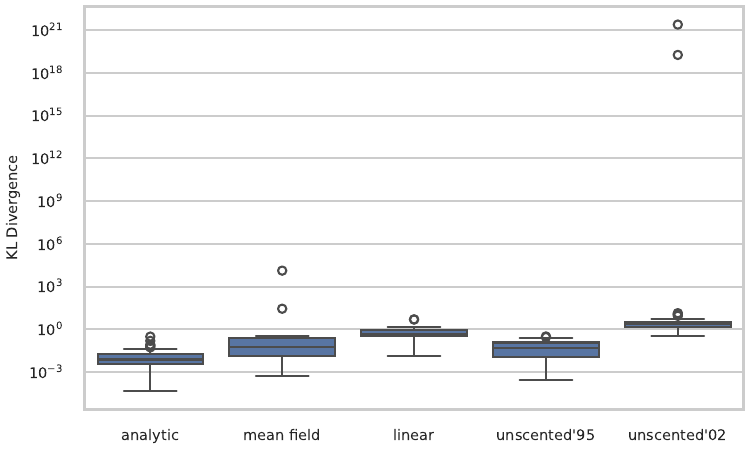}
\end{center}
\caption{Comparison of goodness of approximation (lower KL divergence is better) for all random neural networks, grouped by approximation method, in the medium input variance scenario.}
\end{figure}
\begin{figure}[H]
\begin{center}
  \includegraphics{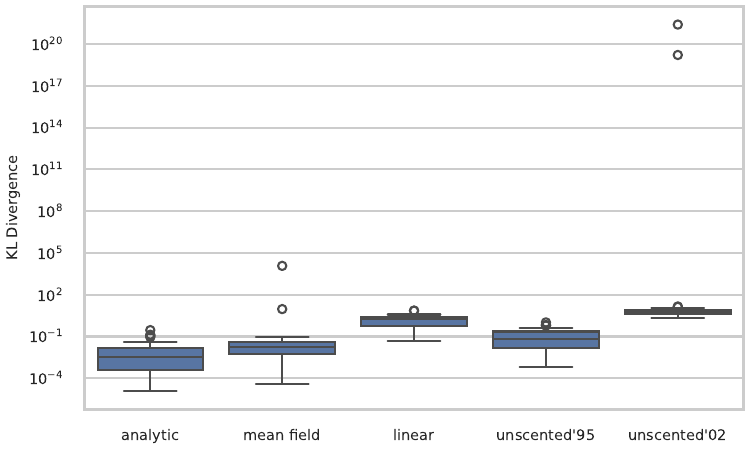}
\end{center}
\caption{Comparison of goodness of approximation (lower KL divergence is better) for all random neural networks, grouped by approximation method, in the large input variance scenario.}
\end{figure}
\clearpage

\clearpage
\section{Comments on computational complexity}
\label{app:computational}
The cost of moment matching, like that of linearized covariance propagation, is cubic in the number of hidden neurons due to the need to evaluate ``sandwich'' covariances such as \(A \Sigma A^\intercal\).
\footnote{
  Contrary to occasional claims that the cost of propagating a covariance matrix scales \emph{quadratically} with the number of hidden neurons \citep[\S5]{akgul_deterministic_2025}.
}
By contrast, the computational complexity of evaluating a neural network is quadratic in the number of hidden neurons.
Thus, theoretically, moment matching is one polynomial order of magnitude slower than point prediction, on the same order as linearization.
For neural networks with high input dimension, moment matching is much faster than Monte Carlo, which scales exponentially with dimension.

\begin{after}
  We performed a simple experiment using Jupyter Notebook's \texttt{\%\%timeit} to compare the computational cost of moment matching and Monte Carlo.
  We compare our method with the ground truth method, which runs quasi-Monte Carlo (QMC) on \(2^{16}\) samples.
  We use GeLU because it has the most complicated covariance expression (see Appendix~\ref{app:gelu}).
  
  On ``wide residual'' networks (defined in Appendix~\ref{app:random-neural-networks}):
  \begin{itemize}
    \item Our method takes 88.4 ms \(\pm\) 5.96 ms.
    \item QMC takes 1.79 s \(\pm\) 83 ms.
  \end{itemize}

  On ``deep residual'' networks (defined in Appendix~\ref{app:random-neural-networks}):
  \begin{itemize}
    \item Our method takes 23.7 ms \(\pm\) 957 \(\mu\)s.
    \item QMC takes 878 ms \(\pm\) 61.8 ms.
  \end{itemize}
\end{after}



\newpage
\section*{NeurIPS Paper Checklist}

The checklist is designed to encourage best practices for responsible machine learning research, addressing issues of reproducibility, transparency, research ethics, and societal impact. Do not remove the checklist: {\bf The papers not including the checklist will be desk rejected.} The checklist should follow the references and follow the (optional) supplemental material.  The checklist does NOT count towards the page
limit. 

Please read the checklist guidelines carefully for information on how to answer these questions. For each question in the checklist:
\begin{itemize}
    \item You should answer \answerYes{}, \answerNo{}, or \answerNA{}.
    \item \answerNA{} means either that the question is Not Applicable for that particular paper or the relevant information is Not Available.
    \item Please provide a short (1--2 sentence) justification right after your answer (even for \answerNA). 
\end{itemize}

{\bf The checklist answers are an integral part of your paper submission.} They are visible to the reviewers, area chairs, senior area chairs, and ethics reviewers. You will also be asked to include it (after eventual revisions) with the final version of your paper, and its final version will be published with the paper.

The reviewers of your paper will be asked to use the checklist as one of the factors in their evaluation. While \answerYes{} is generally preferable to \answerNo{}, it is perfectly acceptable to answer \answerNo{} provided a proper justification is given (e.g., error bars are not reported because it would be too computationally expensive'' or ``we were unable to find the license for the dataset we used''). In general, answering \answerNo{} or \answerNA{} is not grounds for rejection. While the questions are phrased in a binary way, we acknowledge that the true answer is often more nuanced, so please just use your best judgment and write a justification to elaborate. All supporting evidence can appear either in the main paper or the supplemental material, provided in appendix. If you answer \answerYes{} to a question, in the justification please point to the section(s) where related material for the question can be found.

IMPORTANT, please:
\begin{itemize}
    \item {\bf Delete this instruction block, but keep the section heading ``NeurIPS Paper Checklist"},
    \item  {\bf Keep the checklist subsection headings, questions/answers and guidelines below.}
    \item {\bf Do not modify the questions and only use the provided macros for your answers}.
\end{itemize}


\begin{enumerate}

\item {\bf Claims}
    \item[] Question: Do the main claims made in the abstract and introduction accurately reflect the paper's contributions and scope?
    \item[] Answer: \answerYes{} 
    \item[] Justification: The theoretical convergence claim in particular is supported by an experiment.

\item {\bf Limitations}
    \item[] Question: Does the paper discuss the limitations of the work performed by the authors?
    \item[] Answer: \answerYes{} 
    \item[] Justification: The main limitations of this work are that 1) it is theoretical rather than applied, and 2) it leaves softmax-type activations as an open problem.

\item {\bf Theory assumptions and proofs}
    \item[] Question: For each theoretical result, does the paper provide the full set of assumptions and a complete (and correct) proof?
    \item[] Answer: \answerYes{} 
    \item[] Justification: Proofs are in the appendix.

    \item {\bf Experimental result reproducibility}
    \item[] Question: Does the paper fully disclose all the information needed to reproduce the main experimental results of the paper to the extent that it affects the main claims and/or conclusions of the paper (regardless of whether the code and data are provided or not)?
    \item[] Answer: \answerYes{} 
    \item[] Justification: We give detailed instructions as well as Python code.

\item {\bf Open access to data and code}
    \item[] Question: Does the paper provide open access to the data and code, with sufficient instructions to faithfully reproduce the main experimental results, as described in supplemental material?
    \item[] Answer: \answerYes{} 
    \item[] Justification: We provide code in the supplement.

\item {\bf Experimental setting/details}
    \item[] Question: Does the paper specify all the training and test details (e.g., data splits, hyperparameters, how they were chosen, type of optimizer) necessary to understand the results?
    \item[] Answer: \answerYes{} 
    \item[] Justification: This is specified in the methodology appendices.

\item {\bf Experiment statistical significance}
    \item[] Question: Does the paper report error bars suitably and correctly defined or other appropriate information about the statistical significance of the experiments?
    \item[] Answer: \answerYes{} 
    \item[] Justification: Where applicable, we report uncertainty over Monte Carlo or quasi-Monte Carlo repetitions.

\item {\bf Experiments compute resources}
    \item[] Question: For each experiment, does the paper provide sufficient information on the computer resources (type of compute workers, memory, time of execution) needed to reproduce the experiments?
    \item[] Answer: \answerYes{} 
    \item[] Justification: All experiments are laptop-scale, and compute resources are described in the appendices.
    
\item {\bf Code of ethics}
    \item[] Question: Does the research conducted in the paper conform, in every respect, with the NeurIPS Code of Ethics \url{https://neurips.cc/public/EthicsGuidelines}?
    \item[] Answer: \answerYes{} 
    \item[] Justification: This research does not involve human-subjects research or new data collection.

\item {\bf Broader impacts}
    \item[] Question: Does the paper discuss both potential positive societal impacts and negative societal impacts of the work performed?
    \item[] Answer: \answerNA{} 
    \item[] Justification: We believe that the technical contribution of this paper does not have immediate consequences for society.
    
\item {\bf Safeguards}
    \item[] Question: Does the paper describe safeguards that have been put in place for responsible release of data or models that have a high risk for misuse (e.g., pre-trained language models, image generators, or scraped datasets)?
    \item[] Answer: \answerNA{} 
    \item[] Justification: We don't use pre-trained language models, image generators, or scraped datasets.

\item {\bf Licenses for existing assets}
    \item[] Question: Are the creators or original owners of assets (e.g., code, data, models), used in the paper, properly credited and are the license and terms of use explicitly mentioned and properly respected?
    \item[] Answer: \answerYes{} 
    \item[] Justification: This information is stated in the methodological appendix where the asset is introduced.

\item {\bf New assets}
    \item[] Question: Are new assets introduced in the paper well documented and is the documentation provided alongside the assets?
    \item[] Answer: \answerYes{} 
    \item[] Justification: Assets include a large collection of experimental results. We explain how they are generated and provide code to reproduce them.

\item {\bf Crowdsourcing and research with human subjects}
    \item[] Question: For crowdsourcing experiments and research with human subjects, does the paper include the full text of instructions given to participants and screenshots, if applicable, as well as details about compensation (if any)? 
    \item[] Answer: \answerNA{} 
    \item[] Justification: The paper does not involve crowdsourcing nor research with human subjects.

\item {\bf Institutional review board (IRB) approvals or equivalent for research with human subjects}
    \item[] Question: Does the paper describe potential risks incurred by study participants, whether such risks were disclosed to the subjects, and whether Institutional Review Board (IRB) approvals (or an equivalent approval/review based on the requirements of your country or institution) were obtained?
    \item[] Answer: \answerNA{} 
    \item[] Justification: The paper does not involve crowdsourcing or research with human subjects.

\item {\bf Declaration of LLM usage}
    \item[] Question: Does the paper describe the usage of LLMs if it is an important, original, or non-standard component of the core methods in this research? Note that if the LLM is used only for writing, editing, or formatting purposes and does \emph{not} impact the core methodology, scientific rigor, or originality of the research, declaration is not required.
    \item[] Answer: \answerNA{} 
    \item[] Justification: The core method development in this research does not involve LLMs as any important, original, or non-standard components.

\end{enumerate}

\end{document}